\documentclass[letterpaper, 10pt journal, twoside]{IEEEtran}

\IEEEoverridecommandlockouts                              

\usepackage{amsmath,amssymb,amsfonts,amscd,dsfont,amsthm} 
\usepackage{graphicx,tabularx,adjustbox}
\graphicspath{{./fig/}}
\usepackage[font={small},skip=4pt]{caption}
\usepackage{xcolor,colortbl}
\usepackage{booktabs}
\usepackage{adjustbox}
\usepackage{multirow}
\usepackage[normalem]{ulem}
\usepackage{microtype}
\usepackage[breaklinks=true, colorlinks, bookmarks=true]{hyperref}

\newcommand{\bfs}{\mathbf{s}}

\newcommand{\bfu}{\mathbf{u}}
\newcommand{\bfv}{\mathbf{v}}
\newcommand{\bfw}{\mathbf{w}}
\newcommand{\bfx}{\mathbf{x}}
\newcommand{\bfy}{\mathbf{y}}
\newcommand{\bfz}{\mathbf{z}}

\newcommand{\bfmu}{\boldsymbol{\mu}}

\newcommand{\bfI}{\mathbf{I}}

\newcommand{\bbE}{\mathbb{E}}

\newcommand{\bbR}{\mathbb{R}}

\newcommand{\calD}{\mathcal{D}}

\newcommand{\calJ}{\mathcal{J}}
\newcommand{\calK}{\mathcal{K}}
\newcommand{\calL}{\mathcal{L}}

\newcommand{\calN}{\mathcal{N}}
\newcommand{\calO}{\mathcal{O}}

\newtheorem{proposition}{Proposition}
\newtheorem{lemma}{Lemma}
\newtheorem{definition}{Definition}
\newtheorem{definition*}{Definition}
\newtheorem{problem*}{Problem}

\definecolor{LightCyan}{rgb}{0.88,1,1}

\DeclareMathOperator*{\diag}{diag}
\DeclareMathOperator*{\per}{per}
\DeclareMathOperator*{\KL}{KL}

\newcommand{\myparagraph}[1]{\vspace*{0.3ex}\noindent\textbf{#1 }}
\newcommand{\prl}[1]{\left(#1\right)}

\newcommand{\revision}[1]{{\color{blue}#1}}

\let\titleold\title
\renewcommand{\title}[1]{\titleold{#1}\newcommand{\thetitle}{#1}}
\def\maketitlesupplementary
   {
   \newpage
       \twocolumn[
        \centering
        \Large
        \textbf{\thetitle}\\
        \vspace{0.5em}Supplementary Material \\
        \vspace{1.0em}
       ] 
   }

\begin{document}

\title{
PKF: Probabilistic Data Association Kalman Filter for Multi-Object Tracking
}

\author{Hanwen Cao$^{1}$, George J. Pappas$^{2}$, Nikolay Atanasov$^{1}$ 
\thanks{Manuscript received: April, 16, 2025; Revised July, 14, 2025; Accepted August, 25, 2025.} 
\thanks{This paper was recommended for publication by Editor Lucia Pallottino upon evaluation of the Associate Editor and Reviewers' comments.
We gratefully acknowledge support from NSF FRR CAREER 2045945 and ARL DCIST CRA W911NF-17-2-0181.}
\thanks{$^{1}$Hanwen Cao and Nikolay Atanasov are with the Department of Electrical and Computer Engineering, University of California San Diego, La Jolla, CA 92093, USA, e-mails: {\tt\small\{h1cao,\allowbreak natanasov\}@ucsd.edu}.}%
\thanks{$^{2}$George J. Pappas is with the
Department of Electrical and Systems Engineering, University of Pennsylvania, Philadelphia, PA 19104, USA, e-mail: {\tt\small pappasg@seas.upenn.edu}.}%
\thanks{Digital Object Identifier (DOI): see top of this page.}
}


\markboth{IEEE Robotics and Automation Letters. Preprint Version. Accepted August, 2025}
{Cao \MakeLowercase{\textit{et al.}}: PKF: Probabilistic Data Association Kalman Filter}

\maketitle

\begin{abstract}

In this paper, we derive a new Kalman filter (KF) with probabilistic data association between measurements and states. We formulate a variational inference problem to approximate the posterior density of the state conditioned on the measurement data. We view the unknown data association as a latent variable and apply Expectation Maximization (EM) to obtain a filter with the update step in the same form as the Kalman filter but with an expanded measurement vector of all potential associations. We show that the association probabilities can be computed as permanents of matrices with measurement likelihood entries. We name our probabilistic data association Kalman filter the \emph{PKF} with P emphasizing both the \emph{\textbf{p}robabilistic} nature of the data association and the matrix \emph{\textbf{p}ermanent} computation of the association weights. 
We compare PKF with the well-established Probabilistic Multi-Hypothesis Tracking (PMHT) and Joint Probabilistic Data Association Filter (JPDAF) in both theory and simulated experiments. The experiments show that we can achieve lower tracking errors than both.
We also demonstrate the effectiveness of our filter in multi-object tracking (MOT) on multiple real-world datasets, including MOT17, MOT20, and DanceTrack. We can achieve comparable tracking results with previous KF-based methods without using velocities or doing multi-stage data association and remain real-time. 
We further show that our PKF can serve as a backbone for other KF-based trackers by applying it to a method that uses varieties of features for association, and improving its results.

\end{abstract}


\section{Introduction}
\label{sec:intro}

\IEEEPARstart{I}{n} estimation tasks where ambiguity exists between measurements and variables of interest, probabilistic data association can prevent catastrophic estimation failures.
For example, in multi-object tracking (MOT), there can be a lot of occlusions, causing high ambiguity in data association. An illustration of how probabilistic data association improves MOT is shown in Figure~\ref{fig:teaser}. Methods utilizing the Kalman filter (KF)~\cite{sort,bytetrack,ocsort,yang2024hybrid} have achieved outstanding performance in MOT but have not considered the impact of probabilistic data association on the tracking process.

We derive a new form of Kalman filter with probabilistic data association. Previous works~\cite{barfoot2020exactly,Cao_MultiRobotSLAM_RAL24} show that Kalman filter can be derived using variational inference (VI) by maximizing the evidence lower bound (ELBO) of the input and measurement likelihood. To deal with ambiguous associations, we approach the VI problem using Expectation-Maximization (EM) with the data association as the latent variable. In the E-step, we show that association weights can be computed as \emph{permanents} of matrices with measurement likelihood entries using accelerated algorithms \cite{nijenhuis2014combinatorial,huber2008fast}. We also show that the weight computation can be extended to cases with missing or false detections. In the M-step, optimizing the EM objective leads to a Kalman filter with  the usual prediction and update steps but with an extended measurement vector and noise covariance matrix capturing all possible data associations. Our formulation is related to Probabilistic Multi-Hypothesis Tracking (PMHT)~\cite{pmht} and the Joint Probabilistic Data Association Filter (JPDAF) \cite{jpdaf}. PMHT assumes independence of measurements and tracks in both E-step and M-step, which can easily be violated in some problems.
The JPDAF computes a posterior given each associated measurement, i.e., maximizes the ELBO of each measurement, and then computes the weighted average of the posteriors. Our filter, instead, directly optimizes the overall ELBO weighted by the data association weights. We show that our algorithm achieves lower tracking errors than both PMHT and JPDAF, while running at comparable speeds to JPDAF. 

\begin{figure}[t]
    \centering
    \includegraphics[width=\linewidth]{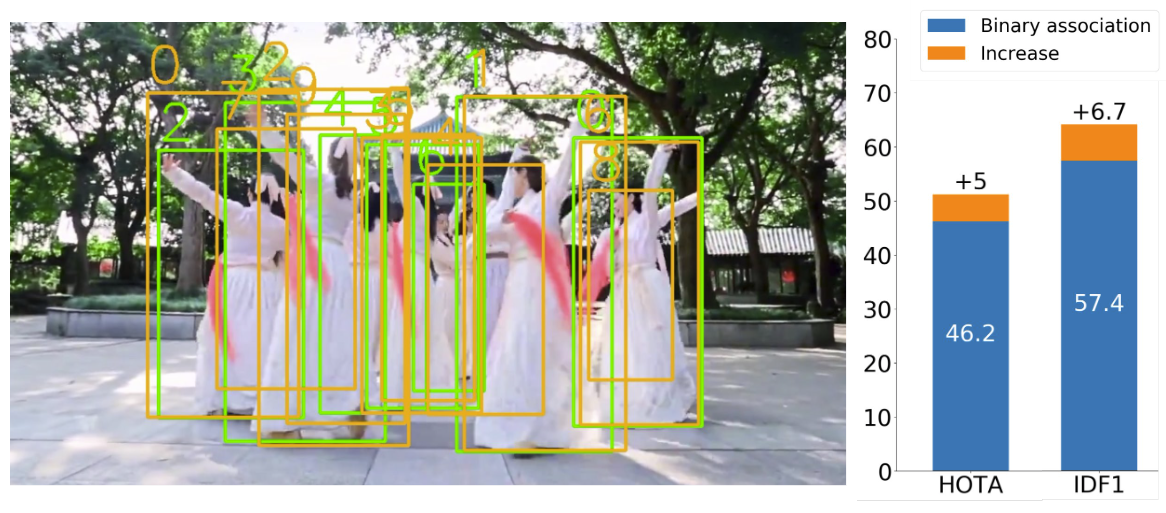}
    \caption{Example scene with high ambiguity in DanceTrack~\cite{dancetrack} with \textcolor{green}{green} box detections and \textcolor{orange}{orange} box tracks. Applying our probabilistic data association Kalman filter on this sequence significantly increases the association quality according to IDF1 \cite{IDF1}, 
    and the combined HOTA metric \cite{hota}.}
    \label{fig:teaser}
\end{figure}


We also apply our filter to MOT 
and test our algorithm on multiple real-world datasets, including MOT17~\cite{MOT17}, MOT20~\cite{MOT20}, and DanceTrack~\cite{dancetrack}. Without using techniques like multi-stage data association, KF re-updates, object recovery, or considering velocities during association, our method achieves comparable performance to previous KF-based methods \cite{bytetrack,ocsort} while maintaining almost the same inference speed. We further show the compatibility of our PKF to varieties of features for data association by applying it to Hybrid-SORT \cite{yang2024hybrid}, which uses the techniques from \cite{bytetrack,ocsort} in addition to using detection scores and neural features. By replacing the KF in \cite{yang2024hybrid} with our PKF, we can further improve its results. 
Our contributions are summarized as follows.

\begin{itemize}
    \item We formulate state estimation with ambiguous measurements as a VI problem and use EM to derive the PKF, a new Kalman filter with probabilistic data association.


    \item We show that association probabilities can be computed using matrix permanents, which can accelerate the association weight computation. 
    
    \item We demonstrate the effectiveness of the PKF in comparison to PMHT and JPDAF on simulated data. On multiple real-world MOT datasets, PKF 
    achieves comparable results with previous state-of-the-art algorithms by associating only bounding boxes, and can improve the results by replacing the KF part, 
    and runs at 250+ fps on a single laptop CPU, given offline detections\footnote{Code is at \url{https://github.com/hwcao17/pkf}.}.

\end{itemize}

\section{Related work}
\label{sec:related_work}


\subsection{Probabilistic data association}
Data association is a key challenge in estimation problems where measurements need to be related to estimated variables. The most straightforward association approach is by a nearest neighbor rule based on some distance metric, such as Euclidean or Mahalonobis \cite{kaess2008isam}. However, nearest-neighbor associations are prone to mistakes. The joint compatibility branch and bound (JCBB) algorithm \cite{jcbb} computes a joint Mahalonobis distance and performs a joint $\chi^2$ test for the association event to obtain robust associations. 
To deal with ambiguity, Bar-Shalom and Tse developed the probabilistic data association filter (PDAF) \cite{pdaf} for tracking single objects in the presence of clutter measurements. In the Kalman filter update step, the innovation terms are computed for each of the possible associations and are averaged weighted by the association probabilities. The performance of PDAF can degrade in the multi-object case because its separate association process can cause multiple tracks to latch onto the same object. The JPDAF \cite{jpdaf} extends the PDAF by associating all measurements and tracks jointly. In each association event, a measurement can be associated with at most one track, and vice versa. 
Musicki et al.~\cite{IPDA} develop an integrated probabilistic data association (IPDA) method using a Markov chain modeling probabilistic data association, track initialization, and termination. JIPDA~\cite{JIPDA} extends IPDA~\cite{IPDA} to the multi-object case by performing data association jointly. There are also improvements on the original JPDAF~\cite{jpdaf} like avoiding coalescence~\cite{blom2000probabilistic} and speeding up the inference \cite{yang2018linear,rezatofighi2015joint}. Meyer et al. \cite{meyer2018message} approach probabilistic data association using a factor graph where the variables (including objects and association events) and functions (motion, observation, et al.) are nodes and the edges model their relationships. The association probabilities can be computed efficiently by message passing in the graph. A survey of data association techniques can be found in \cite{survey}. Besides MOT, probabilistic data association is also applied in other estimation tasks like biological data processing~\cite{bio_tracking} and simultaneous localization and mapping \cite{bowman2017probabilistic,doherty2019multimodal}.

\subsection{Multi-object tracking}
State-of-the-art MOT algorithms are dominated by the tracking-by-detection approach, which consists of two steps: object detection \cite{faster_rcnn,yolo} and data association. SORT~\cite{sort} uses the Hungarian algorithm~\cite{hungarian} to associate bounding-box object detections and updates the object tracks using the Kalman filter to achieve real-time online tracking. DeepSORT~\cite{deepsort} extends SORT~\cite{sort} by introducing deep visual features for data association. ByteTrack~\cite{bytetrack} makes an observation that the detectors can make imperfect predictions in complex scenes and proposes a second round of data association for low-confidence detections. OC-SORT~\cite{ocsort} analyzes the limitations of SORT~\cite{ocsort}, namely sensitivity to state noise and temporal error magnification, and proposes an observation-centric re-update step and object momentum for data association.

A second class of algorithms for MOT performs tracking-by-regression. Feichtenhofer et al.~\cite{feichtenhofer2017detect} use a correlation function between the feature maps of two frames to predict the transformation of bounding boxes in different frames. Bergmann et al.~\cite{bergmann2019tracking} uses the regression head of Faster-RCNN~\cite{faster_rcnn} to predict bounding-box offsets from the previous image to the current image. CenterTrack \cite{centertrack} adopts the same network architecture as CenterNet \cite{centernet} but takes two consecutive images to predict object offsets between frames. Braso and Leal-Taixe \cite{braso2020learning} approach MOT with a graph neural network, which stores object features in the nodes and predicts whether two bounding boxes in different frames are the same object by edge features. 

A third set of methods performs tracking-by-attention \cite{trackformer} using transformer models \cite{transformer} for joint detection and tracking based on new object queries and tracked object queries, which solves the creation and association of objects implicitly. Although these methods utilize different techniques for data association, either explicit or implicit, they all consider deterministic data association. 


\section{Problem formulation}
\label{sec:problem_formulation}

Consider $N$ time-varying variables that need to be estimated, denoted as $\mathbf{x}_{t,j} \in \mathbb{R}^n$, $j=1, \ldots, N$. We refer to their concatenation into a single vector at time $t$:
\begin{equation}
    \bfx_t = \begin{bmatrix}
        \bfx_{t,1}^\top & \cdots & \bfx_{t,N}^\top
    \end{bmatrix}^\top \in \bbR^{nN}
\end{equation}
as the state of a dynamical system. Suppose that $\bfx_t$ evolves according to a discrete-time linear motion model:
\begin{equation}\label{eq:linear_motion_model}
    \bfx_{t+1} = F \bfx_t + G \bfu_t + \bfw_t, \qquad \bfw_t \sim \calN(\mathbf{0}, W),
\end{equation}
where $\bfu_t\in \bbR^l$ is the known input and $\bfw_t$ is zero-mean Gaussian noise with covariance $W$. At each time $t$, we receive $M_t$ measurements
\begin{equation}
    \bfz_t = \begin{bmatrix}
        \bfz_{t,1}^\top & \cdots & \bfz_{t,M_t}^\top
    \end{bmatrix}^\top \in \bbR^{mM_t}.
\end{equation}
The data association between $\bfz_t$ and $\bfx_t$ is defined as follows.

\begin{definition}
\label{def:data_association}
The \emph{data association} of $\bfz_t$ to $\bfx_t$ is a function $\delta_t: \{1,\ldots,M_t\} \to \{0,1,\ldots,N\}$ that either associates an element $\bfz_{t,k}$ to an element $\bfx_{t,\delta_t(k)}$ or indicates via $\delta_t(k) = 0$ that there is no matching element in $\bfx_t$.
\end{definition}

Given association $\delta_t(k)=j$, the measurement model is 
%
\begin{equation}\label{eq:KF_obs_model}
    \bfz_{t,k} = H \bfx_{t,j} + \bfv_t, \qquad \bfv_t \sim \calN(\mathbf{0}, V),
\end{equation}
where $\bfv_t$ is zero-mean Gaussian noise with covariance $V$. 
%
%

\begin{problem*}
Given an estimate of the state $\bfx_t$ at time $t$, an input $\bfu_t$, and measurements $\bfz_{t+1}$, obtain an estimate of $\bfx_{t+1}$ with \textit{unknown association} between the measurements $\bfz_{t+1,k}$ and the variables $\bfx_{t+1,j}$.
\end{problem*}


\section{Methodology}
\label{sec:methodology}

We formulate a Kalman filter with probabilistic data association to solve the estimation problem defined in Sec.~\ref{sec:problem_formulation}.

\subsection{Data association}




Following Definition~\ref{def:data_association}, $\delta_t$ is the data association at time $t$ indicating that measurement $\bfz_{t,k}$ is assigned to variable $\bfx_{t,\delta_t(k)}$.
We denote the set of all possible association functions at time $t$ as $\calD_t = \{ \delta_t: \{1, ..., M_t\} \to \{1, ..., N\} \}$ 
and the set of data associations across all times as $\mathcal{D} = \cup_{t=1}^T \calD_t$.

We consider a probabilistic setting, in which we are given a prior probability density $p(\bfx_t)$ of the state $\bfx_t$, and aim to compute the posterior density $p(\bfx_{t+1} | \bfz_{t+1}, \bfu_t)$ of $\bfx_{t+1}$ conditioned on the input $\bfu_t$ and the measurements $\bfz_{t+1}$. At each time $t$, we treat the data association $\delta_t$ as a random variable with a uniform prior $p(\delta_t)$ independent of $\bfx_t$. Using Bayes' rule and assuming the measurements are mutually independent conditioned on the variables that generated them, the density of $\delta_t$ conditioned on $\bfz_t$ 
%
satisfies:
\begin{equation}\label{eq:assoc_prob}
    p(\delta_t|\bfz_t) \! = \!\frac{p(\bfz_t |\delta_t) p(\delta_t)}{p(\bfz_t)} 
    \!\propto\! p(\bfz_t | \delta_t)  
    \!=\! \prod_{k=1}^{M_t} \!p\left(\bfz_{t,k} | \delta_t(k)\right).
\end{equation}
%
%
Let $p_{\calN}(\cdot; \bfmu, \Sigma)$ denote the density of a Gaussian distribution with mean $\bfmu$ and covariance $\Sigma$. Assuming Gaussian prior $p(\bfx_{t,\delta_t(k)}) \allowbreak =  p_{\calN}(\bfx_{t,\delta_t(k)}; \bfmu_{t,\delta_t(k)}, \Sigma_{t,\delta_t(k)})$ and measurement $p(\bfz_{t,k} | \bfx_{t,\delta_t(k)}) = p_{\calN}(\bfz_{t,k}; H\bfx_{t,\delta_t(k)}, V)$, we have:
\begin{equation*}
    p(\bfz_{t,k} | \delta_t(k)) = p_{\calN}( \bfz_{t,k};\! H\boldsymbol{\mu}_{t,\delta_t(k)}, H \Sigma_{t,\delta_t(k)} H^\top +  V ).
\end{equation*}
Consider the likelihood that a specific measurement $\bfz_{t,k}$ is generated by $\bfx_{t,j}$, denoted by $w^{t}_{k,j}$. Using \eqref{eq:assoc_prob}, we have:
\begin{align} 
    w^{t}_{k,j} = \!\!\!\! \sum_{\delta_t \in \calD_t(k,j)} \!\! p(\delta_t | \bfz_t) = \!\!\!\! \sum_{\delta_t \in \calD_t(k,j)} \prod_{r=1}^{M_t} p\left(\bfz_{t,r} | \delta_t(r)\right),
    \label{eq:weight_def} 
\end{align}
where $\calD_t(k,j)$ is the set of data association functions that assign measurement $\bfz_{t,k}$ to variable $\bfx_{t,j}$. We show that $w^{t}_{k,j}$ can be computed using the \emph{permanent} of a matrix $Q$ containing the measurement likelihood $p(\bfz_{t,k} | \bfx_{t,j})$ as its entries $Q(k,j)$. 

\begin{definition}
The permanent of a matrix $Q = \left[Q(k, j)\right] \in \mathbb{R}^{M \times N}$ with $M \leq N$ is:
\begin{equation}\label{eq:permanent_def}
    \per(Q) := \sum_{\delta \in \calD}\prod_{k=1}^{M} Q(k, \delta(k)),
\end{equation}
where $\calD$ is the set of injective functions $\delta : \{1,\ldots, M \} \to \{1,\ldots, N \}$.    
\end{definition}

Given this definition, we have the following proposition to compute the data association weights.


\begin{proposition}\label{prop:weight_compute}
    Let $Q^{t} \in \mathbb{R}^{M_t \times N}$ be a matrix with elements $Q^{t}(k, j) = p(\bfz_{t,k} | \delta_t(k)=j) = p(\bfz_{t,k} | \bfx_{t,j})$. The data association weight $w^{t}_{k,j}$ in 
    can be computed as:
    \begin{equation*} \label{eq:weight_compute}
    w^{t}_{k,j} \propto Q^{t}(k, j) \per\prl{Q^{t}_{-kj}},
    \end{equation*}
    where $Q^{t}_{-kj}$ is the matrix $Q^t$ with the $k$-th row and $j$-th column removed.
\end{proposition}

\begin{proof}
    See Supplementary Material Section~\ref{sec:proof_weight_compute}.
\end{proof}

Proposition~\ref{prop:weight_compute} allows us to compute data association weights. Next, we derive our PKF.


\subsection{Probabilistic data association Kalman filter}

To derive a filter, we consider two time steps, namely $t$ and $t+1$, and define a lifted form of the state:
\begin{equation}
    \bar{\bfx} = \begin{bmatrix} \bfx_t^\top & \bfx_{t+1}^\top \end{bmatrix}^\top.
\end{equation}
Given the prior density $p(\bfx_t)$, the input $\bfu_t$, and the measurement $\bfz_{t+1}$, we aim to find an estimate $q(\bar{\bfx})$ of the joint posterior $p(\bar{\bfx} | \bfz_{t+1}, \bfu_{t})$. To determine $q(\bar{\bfx})$, we use variational inference \cite[Ch.~10]{bishop2006pattern}, which maximizes the lower bound of the evidence $\log p(\bfz_{t+1}, \bfu_t)$ with respect to $q(\bar{\bfx})$. To save space, we denote $q(\bar{\bfx})$ by $q$. The evidence can be decomposed as~\cite[Ch.~10]{bishop2006pattern}:
\begin{align}
    \log p(\bfz_{t+1}, \bfu_{t}) &= \KL(q||p(\bar{\bfx} | \bfz_{t+1}, \bfu_{t})) + \calL(q), \label{eq:evidence} \\
    \calL(q) &= \bbE_{q} \bigl[\log \frac{p(\bar{\bfx}, \bfz_{t+1}, \bfu_t)}{q(\bar{\bfx})}\bigr],
\end{align}
where $\calL(q)$ is the ELBO (evidence lower bound) since the Kullback–Leibler divergence $\KL(q||p) \geq 0$. Following standard practice \cite[Ch.~10]{bishop2006pattern}, we optimize the ELBO
%
\begin{equation}
	q^* \! \in \! \arg\max_q  \calL(q) \! = \! \arg\max_q \bbE_{q} \bigl[\log \frac{p(\bar{\bfx}, \bfz_{t+1}, \bfu_t)}{q(\bar{\bfx})}\bigr].
\end{equation}
%
%
To account for probabilistic data association, we formulate the optimization by the EM algorithm \cite[Ch.~9]{bishop2006pattern}. EM determines the maximum likelihood (ML) or maximum a posteriori (MAP) in the presence of unobserved variables. It contains an E-step that computes the expectation w.r.t. the unobserved variables and an M-step that maximizes the log-likelihood. Letting the data association $\delta_{t+1}$ be the unobserved variable, we consider the following problem:
%
\begin{align}
&q^{(i+1)} \in \arg\max_{q} f^{(i)}(q),  \label{eq:EM_problem} \\
&= \arg\max_{q} \mathbb{E}_{\delta_{t+1}} \bbE_{\bar{\bfx}^{(i)}} \! \left[ \calL(q, \delta_{t+1}) \mid \bar{\bfx}^{(i)}, \bfz_{t+1} \right] \notag  \\
&= \arg\max_{q}  \mathbb{E}_{\delta_{t+1}} \bbE_{\bar{\bfx}^{(i)}} \!\! \left[ \bbE_{q} \bigl[\log \! \frac{p(\bar{\bfx}, \bfz_{t+1}, \bfu_t, \delta_{t+1})}{q(\bar{\bfx})} \bigr] \!\! \mid \!\! \bar{\bfx}^{(i)}\!, \bfz_{t+1} \! \right] \notag \\
&= \arg\max_{q} \sum_{k=1}^{M_t} \sum_{j=1}^N  \omega^{t+1}_{kj} \log p(\mathbf{z}_{t+1, k} | \bfx_{t+1, j}), \notag
\end{align}
where the expectation $\bbE_{\bar{\bfx}^{(i)}}$ is with respect to $\bar{\bfx}^{(i)} \sim q^{(i)}(\bar{\bfx})$. EM splits the optimization in \eqref{eq:EM_problem} in two steps. The E-step requires computing the data association likelihood $p(\delta_{t+1} | \bfz_{t+1})$. This can be obtained from \eqref{eq:assoc_prob} and Proposition~\ref{prop:weight_compute}. Given the data association weights $w^{t}_{k,j}$, the M-step performs the optimization in \eqref{eq:EM_problem} to determine $q^{(i+1)}$.


We show that performing the E and M steps for one iteration, with initialization $\bar{\bfx}^{(i)}$ given by the prior state and the predicted state in \eqref{eq:linear_motion_model}, is equivalent to a Kalman filter with probabilistic data association. We first define an expanded observation model that captures all possible ways of generating the measurements at each time $t$:
\begin{equation}\label{eq:expanded_obs_model}
    \bar{\bfz}_{t} = \bar{H}_t \bfx_t + \bar{\bfv}_t, \qquad \bar{\bfv}_{t} \sim \calN(\mathbf{0}, \bar{V}_t),
\end{equation}
where $\bar{\bfz}_t = (I_{M_t} \otimes \mathbf{1}_{N} \otimes I_{m}) \bfz_t \in \bbR^{mM_tN}$, $\bar{H}_t = \mathbf{1}_{M_t} \otimes \bfI_N \otimes H \in \mathbb{R}^{mM_tN \times nN}$, $\mathbf{1}_N \in \bbR^N$ is a vector with elements equal to one, $\bfI_N \in \bbR^{N \times N}$ is the identity matrix, $\otimes$ is the Kronecker product, and $\bar{\bfv}_t$ is zero-mean Gaussian noise with covariance:
\begin{equation}\label{eq:expanded_obs_cov}
    {\setlength{\arraycolsep}{2pt} 
    \medmuskip = 1mu 
    \bar{V}_t \! = \!\! \begin{bmatrix}
                \bar{V}_{t,1} & \ & \ \\
                \ & \ddots & \ \\
                \ & \ & \bar{V}_{t,M_t}
                \end{bmatrix}\!\!,
    \bar{V}_{t,k} \! = \!\! \begin{bmatrix}
                    V \! / \! w^t_{k,1} & \ & \ \\
                    \ & \ddots & \ \\
                    \ & \ & V \! / \! w^t_{k,N}
                    \end{bmatrix}\!\!,}
\end{equation}
where $w^t_{k,j}$ are the data association weights in \eqref{eq:weight_def}. Using the expanded measurement model in \eqref{eq:expanded_obs_model}, we obtain a Kalman filter with probabilistic data association.

\begin{proposition}
\label{prop:KF}
Given prior $\bfx_t \sim \calN(\bfmu_t, \Sigma_t)$ and input $\bfu_t$, the predicted Gaussian distribution $\calN(\bfmu_{t+1}^+, \Sigma_{t+1}^+)$ of $\bfx_{t+1}$ computed by the Kalman filter with motion model in \eqref{eq:linear_motion_model} has parameters:
\begin{equation}
\begin{aligned}
    \bfmu_{t+1}^+ &= F \bfmu_t + G\bfu_t,\\
    \Sigma_{t+1}^+ &= F \Sigma_t F^\top + W.
\end{aligned}
\end{equation}
Given measurements $\bfz_{t+1}$, the updated Gaussian distribution $\calN(\bfmu_{t+1}, \Sigma_{t+1})$ of $\bfx_{t+1}$ conditioned on $\bfz_{t+1}$ computed by the Kalman filter with probabilistic data association is obtained from the expanded measurement model in \eqref{eq:expanded_obs_model} as:
\begin{equation}
\begin{aligned}
    \bfmu_{t+1} &= \bfmu_{t+1}^+ + \bar{K}_{t+1} (\bar{\bfz}_{t+1} - \bar{H}_{t+1} \bfmu_{t+1}^+),\\
    \Sigma_{t+1} &= (I - \bar{K}_{t+1} \bar{H}_{t+1}) \Sigma_{t+1}^+,
\end{aligned}
\end{equation}
where the Kalman gain is:
\begin{equation}
    \bar{K}_{t+1} = \Sigma_{t+1}^+ \bar{H}_{t+1}^\top \prl{\bar{H}_{t+1} \Sigma_{t+1}^+\bar{H}_{t+1}^\top + \bar{V}_{t+1}}^{-1}.
    \label{eq:Kalman_gain}
\end{equation}
\end{proposition}

\begin{proof}
    See Supplementary Material Section~\ref{sec:proof_PDA_KF}.
\end{proof}

\subsection{Comparison with PMHT and JPDAF}

We discuss the relationship of PKF to PMHT \cite{pmht} and JPDAF \cite{jpdaf} and compare them in simulated experiments.

\subsubsection{Comparison with PMHT}

The difference between our PKF and PMHT~\cite{pmht} is that PMHT assumes independence in both the E-step (association) and M-step (update), while the PKF does not. During association, PMHT computes the probability of each measurement and object individually without considering if the measurement is already assigned to another object. In contrast, PKF does the association in a joint manner by evaluating probabilities of the injective association functions, where a measurement can not be assigned to another object after being assigned to one. During the update step, PMHT updates each object individually, while PKF updates objects jointly, i.e., there are cross-correlation blocks is the object covariances.

\subsubsection{Comparison with JPDAF}
\label{sec:update_scheme}
The major difference is in the update step, we first show the update scheme difference, then show that the association weights in JPDAF can be computed using Proposition \ref{prop:weight_compute}.

\myparagraph{Update scheme} Both JPDAF~\cite{jpdaf} and PKF can be decomposed into two stages. The JPDAF first computes a posterior mean as the normal Kalman filter with each associated measurement, and then averages the posterior means with the association weights and computes the covariance correspondingly. Instead, our PKF first constructs an expanded measurement model and then performs a single update in the same form as the normal Kalman filter but with an expanded measurement vector. The JPDAF maximizes the ELBO with regard to each associated measurement and then averages those estimates, while our PKF maximizes the total weight-averaged ELBO, which does not necessarily maximize each measurement's individual ELBO but may lead to a higher total ELBO. This difference is underscored by the lower tracking errors in the simulations in Sec.~\ref{sec:jpda_sim_experiments}. 

Both algorithms can be performed in a coupling or decoupling way. In the coupling way, all the objects are stored in one filter and updated jointly. In the decoupling way, each object is stored and updated separately, i.e., setting the cross-correlation blocks of $\Sigma_t$ in Proposition~\ref{prop:KF} to $0$. 
In the experiments, we do decoupling updates just like JPDAF. 
Regarding the complexity, the Kalman gain in JPDAF is the same as the normal Kalman filter but JPDAF needs all the measurement residuals to update the mean and covariance. In PKF, the complexity of the Kalman gain computation increases due to the expanded measurement model. The original Kalman filter has complexity $\calO(m^{2.4} + n^2)$ \cite{montella2011kalman} in terms of the measurement dimension $m$ and state dimension $n$. Given the number of associated measurements $M_t$, the complexity of PKF is $\calO\left((M_t m)^{2.4} + M_t n^{2} \right)$. 

\myparagraph{Weight computation} Following JPDAF~\cite{jpdaf}, we make the assumption that 1) the number of established objects is known and 2) there can be missing or false measurements (clutter).
Using our notation, the probability of an association event $\delta_t$ in \cite[Eq.~47]{jpdaf} can be written as 
\begin{equation*} 
    p(\delta_t | \bfz_t) \! \propto  \! \prod_{k} \! \left( \frac{1}{\lambda} p(\bfz_{t,k} | \bfx_{t,\delta_t(k)}) \right)^{\tau_k} \! \prod_{j} \! (p^t_{D})^{\sigma_j} \! (1 - p^t_{D})^{1-\sigma_j}
\end{equation*}
where $\lambda$ is the spatial density of the false measurements Poisson pmf, $p^t_D$ is the detection probability, $\tau_k$ is an indicator of whether $\bfz_{t,k}$ is associated in the event, i.e. $\bfz_{t,k}$ is not treated as clutter, and $\sigma_{j}$ is an indicator of whether the object $\bfx_{t,j}$ is associated. Suppose that the number of clutter points in $\delta_t$ is $\Phi_t$, the number of measurements is $M_t$, and the number of objects is $N$. Since the number of established objects is assumed known, we have $M_t - \Phi_t \leq N$. With the number of clutter points fixed, $p(\delta_t | \bfz_t)$ can be simplified as
\begin{equation}\label{eq:simplifiled_even_prob}
\begin{aligned}
    p(\delta_t | \bfz_t) &\propto c_t^{\Phi_t} \prod_{k} p(\bfz_{t,k} | \bfx_{t,\delta_t(k)}), \;\; \delta_t \in \calD_t^{\Phi_t},  \\
    c_t^{\Phi_t} &= \lambda^{-\Phi_t}  (p^t_D)^{M_t  -   \Phi_t}(1-p^t_D)^{N_t -  M_t  +  \Phi_t}, 
\end{aligned}
\end{equation}
where $\calD_t^{\Phi_t}$ is the set of association events that have $\Phi_t$ false measurements. The probability of measurement $\bfz_{t,k}$ being assigned to object $\bfx_{t,j}$ is computed in \cite[Eq.~51]{jpdaf} as
\begin{align}
    w^t_{k,j} \triangleq \!\!\!\!\! \sum_{\delta_t \in \calD_{t}(k,j)} & p(\delta_t | \bfz_t)  = \sum_{\Phi_t}  \sum_{\delta_t \in \calD^{\Phi_t}_{t}(k,j)} \!\!\!\!\!\! p(\delta_t | \bfz_t) \label{eq:jpdaf_weight} \\
    &\propto \sum_{\Phi_t}  \sum_{\delta_t \in \calD^{\Phi_t}_{t}(k,j)} \!\!\!\!\! c_t^{\Phi_t}  \prod_{k}  p(\bfz_{t,k} | \bfx_{t,j}), \notag
\end{align}
where $\calD^{\Phi_t}_{t}(k,j)$ is the set of associations having $\Phi_t$ false measurements and assigning $\bfz_{t,k}$ to $\bfx_{t,j}$. Selecting a series of sets of false measurements of size $\Phi_t=\max\{0, M_t-N\}\ \cdots \ M_t$, for each fixed $\Phi_t$, the second sum on the last line of \eqref{eq:jpdaf_weight} can be computed via Proposition~\ref{prop:weight_compute}.

\subsubsection{Tracking error}
\label{sec:jpda_sim_experiments}

\begin{figure}[!t]
    \centering
    \includegraphics[width=0.8\linewidth]{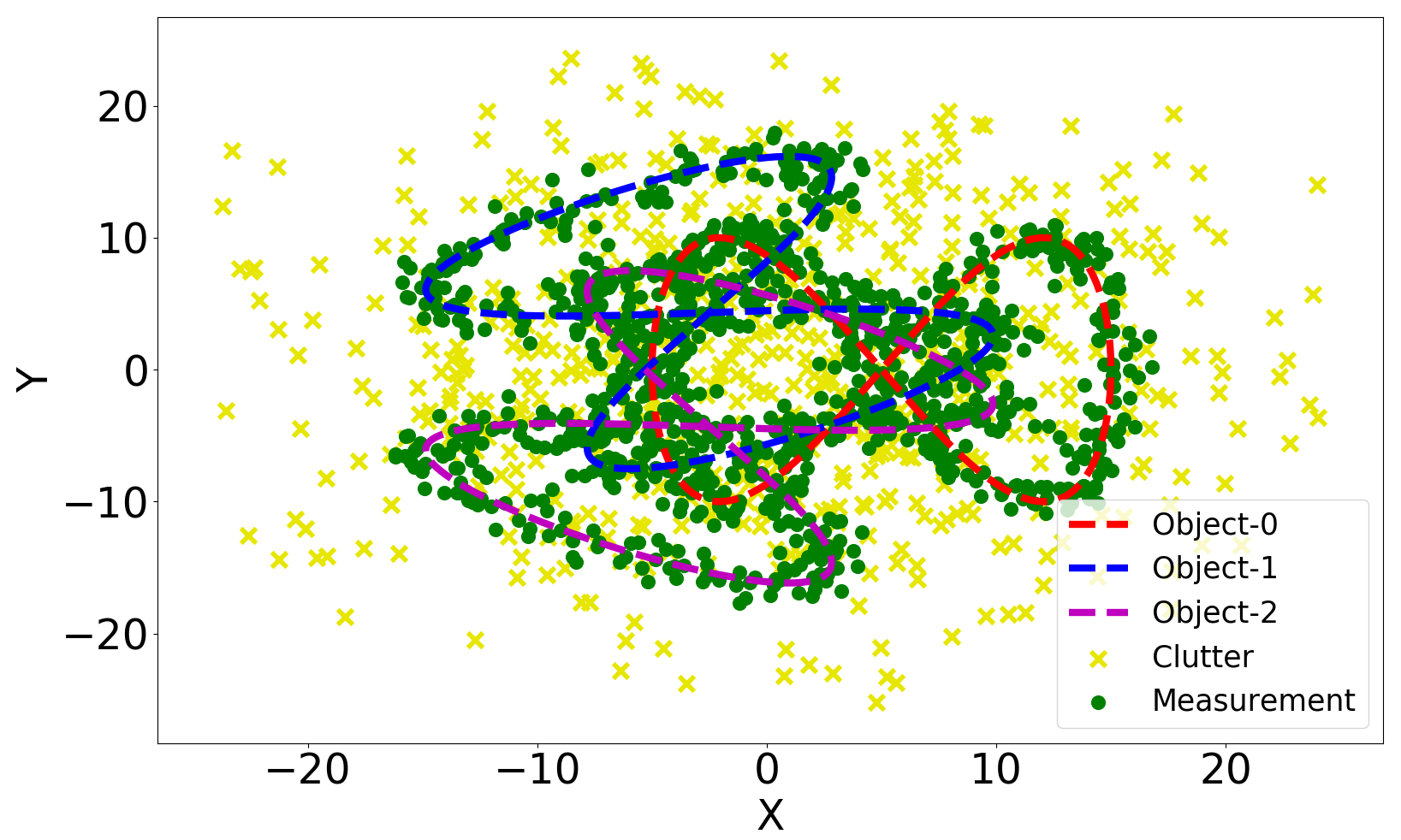}
    \caption{Object tracking simulation showing ground-truth object trajectories (dashed lines), object measurements (green dots), and clutter (false) measurements (yellow crosses).}
    \label{fig:sim_data_visu}
\end{figure}

We compare our PKF to PMHT~\cite{pmht} and JPDAF~\cite{jpdaf} with 2D simulated data, where each object is a point and moves along an 8-shape trajectory. The detection probability of each object is $p^t_D=0.9$, the measurement noise is a zero-mean Gaussian with covariance $\diag(0.75, 0.75)$, and clutter is sampled uniformly in the range of $[-10, 10]^2$ around each ground-truth point. We use $\lambda=0.125$ for the spatial density of the Poisson pmf. A visualization of 3 simulated objects can be found in Figure~\ref{fig:sim_data_visu}.

The errors of tracking 3 and 5 objects in simulation are shown in Table~\ref{tab:sim_results}. PMHT~\cite{pmht} works slightly better than binary ssociation but much worse than PKF and JPDAF~\cite{jpdaf} due to the non-exclusive data association. PKF achieves lower tracking errors than JPDAF implemented by \cite{stonesoup}.
Our intuition as to why PKF achieves lower errors is provided in the first paragraph of Sec.~\ref{sec:update_scheme}. 
We also compare PKF to JPDAF under different levels of measurement noise with 10 objects. In this setting, we increase the detection probability to $0.95$ to avoid frequent tracking failures. The clutter generation remains the same, and we test with measurement noises from $0.2$ to $0.75$. The results can be found in Figure~\ref{fig:sim_10_obj_ablation}. We can see that our filter obtains lower tracking errors in most cases. The number of failed tracks is almost the same as JPDAF~\cite{jpdaf}. 

\subsubsection{Update speed} We compare the update speeds of these methods with different numbers of objects. The results can be found in Table~\ref{tab:update_speed}. The experiments show that PKF can track at a comparable speed with JPDAF.


\begin{table}[!t]
\centering
\caption{Comparison of tracking errors ($l2$ norm in meters).}
\label{tab:sim_results}
\resizebox{\linewidth}{!}{
\begin{tabular}{llllllll}
\hline
$N_{obj}$     & Method      & Obj 1 & Obj 2 & Obj 3 & Obj 4 & Obj 5 & Avg \\ \hline
\multirow{4}{*}{3} & Binary    & 2.16     & 1.36     & 1.80     & -        & -        & 1.78 \\
            & PMHT~\cite{pmht} & 1.50     & 1.29     & 1.55     & -        & -        & 1.45  \\
            & JPDAF~\cite{jpdaf} & 0.70     & 0.60     & 0.65     & -        & -        & 0.65  \\
            & \textbf{PKF}  & \bf{0.70}     & \bf{0.57}     & \bf{0.58}     & -        & -        & \bf{0.62}  \\ \hline
\multirow{4}{*}{5} & Binary & 8.37     & 19.36    & 17.34     & 11.37     & 14.53     & 14.20  \\
        & PMHT~\cite{pmht} & 12.41     & 12.34     & 12.43     & 12.10        & 11.21        & 12.10  \\
        & JPDAF~\cite{jpdaf} & 0.68     & 0.73    & 0.62     & 0.63     & 0.66     & 0.66  \\
        & \textbf{PKF}  & \bf{0.60}     & \bf{0.68}     & \bf{0.58}     & \bf{0.61}     & \bf{0.61}     & \bf{0.62}  \\ \hline
\end{tabular}}
\end{table}

\begin{table}[!t]
\centering
\caption{Update time (ms) per frame given association weights.}
\label{tab:update_speed}
\begin{tabular}{llllll}
\hline
Number of objects      & 3 & 5 & 10 & 20 & Avg \\ \hline
Vanilla Kalman filter    & 0.08     & 0.12     & 0.20     & 0.36    & 0.19  \\
PMHT~\cite{pmht} & 0.11     & 0.16     & 0.28     & 0.55   & 0.28     \\
JPDAF~\cite{jpdaf} & 0.28     & 0.39     & 0.78     & 1.47   & 0.73     \\
PKF  & 0.31     & 0.46     & 0.93     & 1.76   & 0.87      \\ \hline
\end{tabular}
\end{table}

\begin{figure}[!t]
    \centering
    \includegraphics[width=1.0\linewidth]{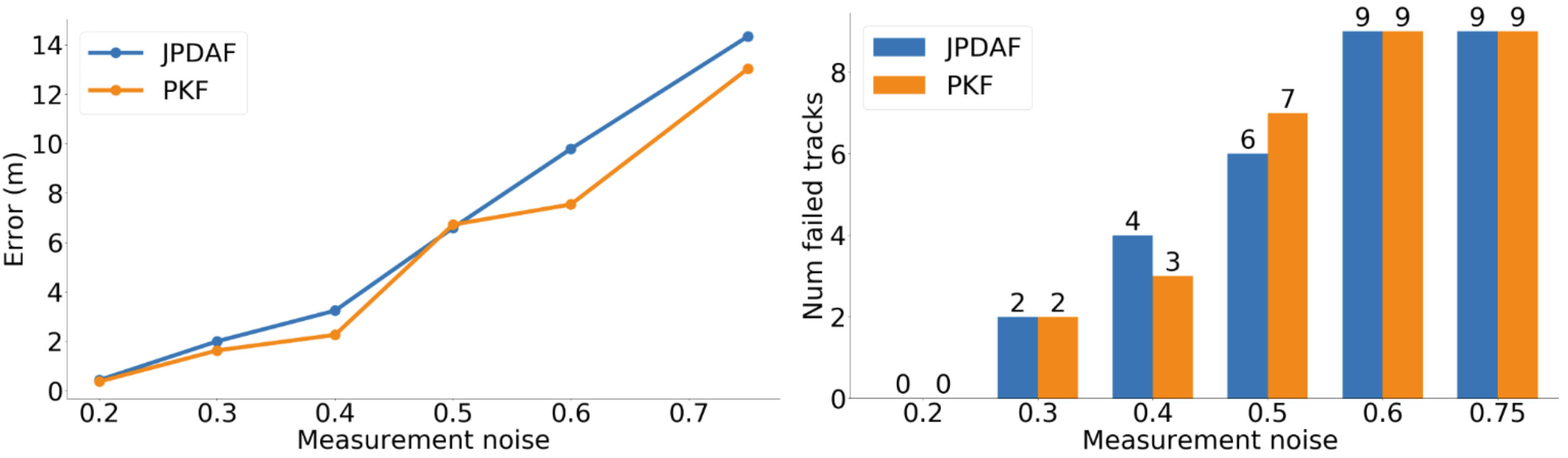}
    \caption{Ablation study with 10 objects in the simulation under different noise scales. Left: tracking errors in meters. Right: number of failed tracks, i.e., those with tracking errors larger than $5$.}
    \label{fig:sim_10_obj_ablation}
\end{figure}

\begin{table*}[!ht]
\centering
\caption{Results on MOT17~\cite{MOT17} testset with private detections. Methods in \textcolor{cyan}{blue} share detections.}
\label{tab:mot17_results}
\resizebox{0.85\linewidth}{!}{
\begin{tabular}{l|lllllllll}
\toprule
Tracker                  & HOTA $\uparrow$ & MOTA $\uparrow$ & IDF1 $\uparrow$ & FP($10^4$) $\downarrow$ & FN($10^4$) $\downarrow$ & IDs $\downarrow$  & Frag $\downarrow$ & AssA $\uparrow$ & AssR $\uparrow$ \\ \midrule
FairMOT~\cite{zhang2021fairmot}                & 59.3 & 73.7 & 72.3 & 2.75       & 11.7       & 3303 & 8073 & 58.0 & 63.6 \\
QDTrack~\cite{QDTrack}                & 53.9 & 68.7 & 66.3 & 2.66       & 14.7       & 3378 & 8091 & 52.7 & 57.2 \\
MOTR~\cite{MOTR}                & 57.2 & 71.9 & 68.4 & 2.11       & 13.6       & 2115 & 3897 & 55.8 & 59.2 \\
TransMOT~\cite{chu2023transmot}                & 61.7 & 76.7 & 75.1 & 3.62       & 9.32       & 2346 & 7719 & 59.9 & 66.5 \\
MeMOT~\cite{cai2022memot}                & 56.9 & 72.5 & 69.0 & 2.72       & 11.5       & 2724 & - & 55.2 & - \\
\rowcolor{LightCyan} ByteTrack~\cite{bytetrack}                & 63.1 & 80.3 & 77.3 & 2.55       & 8.37       & 2196 & 2277 & 62.0 & 68.2 \\
\rowcolor{LightCyan} OC-SORT~\cite{ocsort} & 63.2 & 78.0 & 77.5 & 1.51       & 10.8       & 1950  & 2040 & 63.2 & 67.5 \\
\rowcolor{LightCyan} \textbf{PKF}                  & \textbf{63.3} & \textbf{78.3} & \textbf{77.7} & \textbf{1.56}       & \textbf{10.5}       & \textbf{2130}  & \textbf{2754}    & \textbf{63.4} & \textbf{67.4} \\
\bottomrule
\end{tabular}
}
\end{table*}

\begin{table*}[t]
\centering
\caption{Results on MOT20~\cite{MOT20} testset with private detections. Methods in \textcolor{cyan}{blue} share detections.}
\label{tab:mot20_results}
\resizebox{0.85\linewidth}{!}{
\begin{tabular}{l|lllllllll}
\toprule
Tracker                  & HOTA $\uparrow$ & MOTA $\uparrow$ & IDF1 $\uparrow$ & FP($10^4$) $\downarrow$ & FN($10^4$) $\downarrow$ & IDs $\downarrow$  & Frag $\downarrow$ & AssA $\uparrow$ & AssR $\uparrow$ \\ \midrule
FairMOT~\cite{zhang2021fairmot}                & 54.6 & 61.8 & 67.3 & 10.3       & 8.89       & 5243 & 7874 & 54.7 & 60.7 \\
TransMOT~\cite{chu2023transmot}                & 61.9 & 77.5 & 75.2 & 3.42       & 8.08       & 1615 & 2421 & 60.1 & 66.3 \\
MeMOT~\cite{cai2022memot}                & 54.1 & 63.7 & 66.1 & 4.79       & 13.8       & 1938 & - & 55.0 & - \\
\rowcolor{LightCyan} ByteTrack~\cite{bytetrack}                & 61.3 & 77.8 & 75.2 & 2.62       & 8.76       & 1223 & 1460 & 59.6 & 66.2 \\
\rowcolor{LightCyan} OC-SORT~\cite{ocsort} & 62.1 & 75.5 & 75.9 & 1.80       & 10.8       & 913  & 1198 & 62.0 & 67.5 \\
\rowcolor{LightCyan} \textbf{PKF}                  & \textbf{62.3} & \textbf{75.4} & \textbf{76.3} & \textbf{1.73}       & \textbf{10.9}       & \textbf{980}  & \textbf{1584}    & \textbf{62.7} & \textbf{67.6} \\ 
\bottomrule
\end{tabular}
}
\end{table*}

\section{Application to multi-object tracking}

In this section, we apply our PKF to the MOT task. 

\subsection{Algorithm design}
\label{sec:algorithm_design}

To make the filter work efficiently in practice, we introduce an ambiguity check to avoid outliers and unnecessary computation.
Also, we force the state covariance matrix to be block diagonal so that each track is estimated by a separate filter, although the states are still considered jointly when evaluating the association probabilities. Following SORT~\cite{sort}, we have
$\bfx = [u, v, s, r, \dot{u}, \dot{v}, \dot{s}]^\top$ and $\bfz = [u, v, s, r]^\top$
%
%
where $u$ and $v$ are the horizontal and vertical pixel location of the center of the object box, $s$ and $r$ represent the scale (area) and the aspect ratio of the object’s box, and $\dot{()}$ terms are velocities. Details about the motion and measurement model matrices can be found in Section \ref{sec:motion_measurement_models}.

\myparagraph{Ambiguity check}
Gating is very important in Kalman filtering~\cite{jpdaf}. We propose a simple yet efficient way to find ambiguous measurements in the meantime of removing potential outliers (gating).
Given measurements $\bfz_{t,k}$, $k=1,\ldots, M_t$ and states $\bfx_{t,j}$, $j = 1,\ldots, N$, we first compute a score matrix $S_t \in \bbR^{M_t \times N}$ with entries $[S_t]_{kj}$ equal to the intersection over union (IoU) for each measurement-object bounding-box pair. 
Then, we perform an ambiguity check. For each row $k$ of $S_t$ (corresponding to measurement $\bfz_{t,k}$), we sort the scores from high to low. Supposing the rank is $\bfx_{t,j_1} \ \bfx_{t,j_2} \cdots \bfx_{t,j_N}$, if object $\bfx_{t,j_2}$ has a score over a threshold $\tau_{ambig}$, e.g., $90\%$ the score of object $\bfx_{t,j_1}$, then $\bfx_{t,j_1}$ and $\bfx_{t,j_2}$ are ambiguous objects, and the measurement $\bfz_{t,k}$ is an ambiguous measurement. Then, we compare $\bfx_{t,j_2}$ and $\bfx_{t,j_3}$ and so on. We repeat this for all objects and measurements to obtain an ambiguous set. Measurements and objects with their matches marked as ambiguous will also be added to the ambiguous set. Proposition~\ref{prop:weight_compute} is only applied to the ambiguous set of measurements and objects. 
The rest of the measurements and objects are associated with binary association~\cite{hungarian,jonker-volgenant}.
An illustration of how our proposed probabilistic data association Kalman filter benefits the ambiguous case is shown in Figure~\ref{fig:assoc_illustration}, where binary association causes a wrong ID switch while our proposed association can obtain a correct ID.

We create a filter for each new object (equivalent to enforcing a block-diagonal covariance matrix). This avoids a large state covariance and observation Jacobian $\bar{H}_{t+1}$ in \eqref{eq:Kalman_gain}, which speeds up the tracking especially when the number of objects is large. We still compute the association weight matrix $W$ jointly, each column vector $\bfw_j$ corresponds to a tracked object with measurement association weights. Measurements with a weight larger than a threshold $\tau_{weight}$ are used to update the tracker by Proposition~\ref{prop:KF}, which helps exclude potential outliers and achieve stable computation.

\begin{figure}[t]
    \centering
    \includegraphics[width=\linewidth,trim={0 10pt 0 10pt},clip]{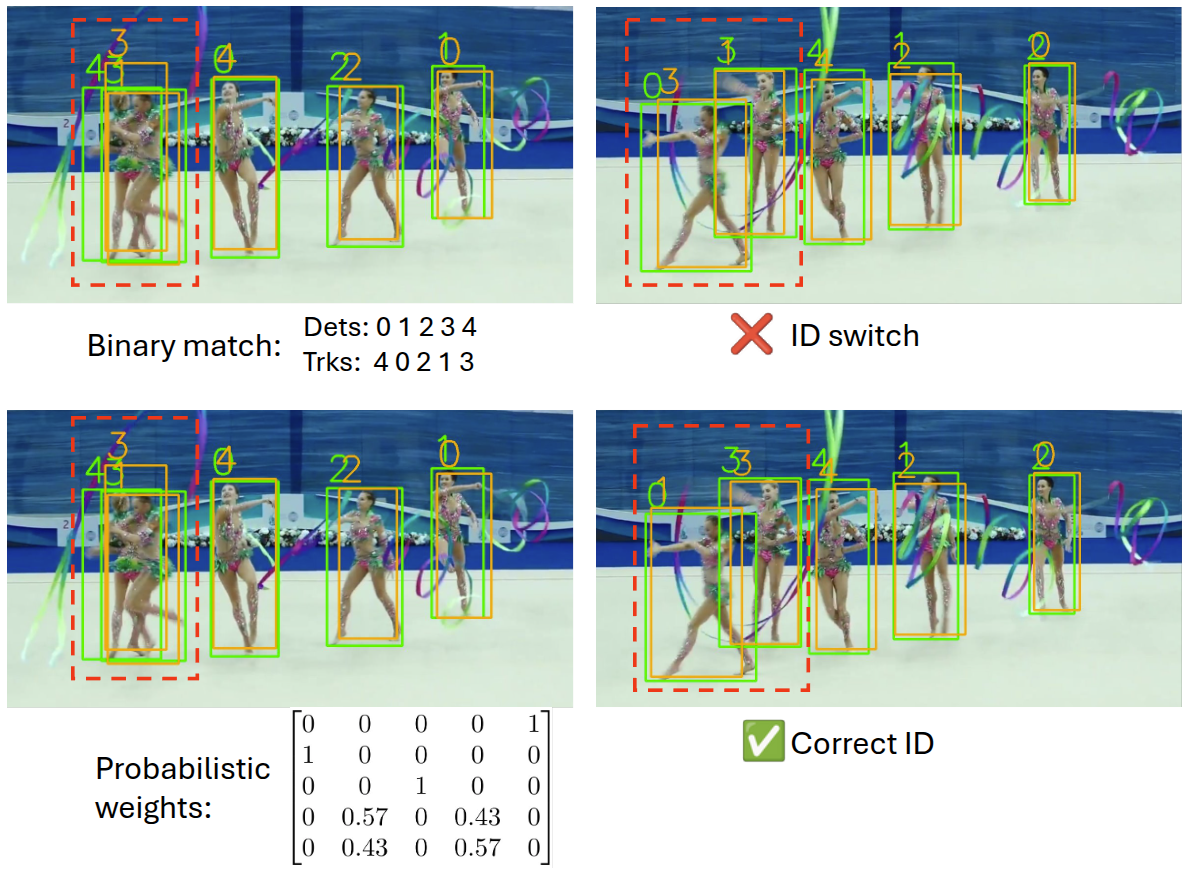}
    \caption{Illustration of our data association. The detections and tracks are shown in \textcolor{green}{green} and \textcolor{orange}{orange}. The first row shows the result of binary association via 
    \cite{jonker-volgenant}. The second row is our proposed data association. 
    In the first image, the detections highly overlap, and our ambiguity check recognizes these as ambiguous (\textcolor{red}{dashed red boxes}). The binary association causes an ID switch, but our method can track the objects correctly, as shown in the third image.
    }
    \label{fig:assoc_illustration}
\end{figure}

\subsection{Evaluation}
\label{sec:evaluation}


\myparagraph{Datasets}
We evaluate PKF on multiple datasets, including MOT17~\cite{MOT17}, MOT20~\cite{MOT20}, and DanceTrack~\cite{dancetrack}. MOT17 and MOT20 focus on crowded pedestrian tracking in public places. DanceTrack is a dancing-scene dataset, where the objects have similar appearances, move fast, highly non-linearly, and cross over each other frequently, thus imposing a higher requirement on data association. 

\myparagraph{Metrics}
We use higher order tracking accuracy (HOTA) \cite{hota} as our main metric because it balances between detection accuracy and association accuracy. We also report the AssA~\cite{hota} and IDF1~\cite{IDF1} metrics, which emphasize data association accuracy and the MOTA~\cite{clearmot} metric, which emphasizes the detection accuracy.

\myparagraph{Implementation}
In performing the data association ambiguity check, we compute the IoU between detections and tracked bounding boxes. To construct the matrix $Q$ in Proposition~\ref{prop:weight_compute}, we treat $UoI := 1/IoU$ as a distance (not an actual distance function mathematically) and compute the conditional probability as $p(\bfz_{t,k} | \bfx_{t,j}) \propto \exp(-\alpha UoI)$, where $\alpha$ is a scaling factor. We found $\alpha=2$ to work well in practice. On MOT17~\cite{MOT17} and DanceTrack~\cite{dancetrack}, the ambiguity check threshold was set to $\tau_{ambig}=0.9$. On MOT20, we set $\tau_{ambig}=0.95$ since the pedestrians are more crowded. The weight threshold was set at $\tau_{weight}=0.25$. We used the publicly available YOLOX weights by ByteTrack~\cite{bytetrack}. Following the common practice of SORT~\cite{sort}, we set the detection confidence threshold at $0.4$ for MOT20 and $0.6$ for other datasets. A new track is created if a detected box has IoU lower than $0.3$ with all the tracks. All experiments were conducted on a laptop with i9-11980HK@2.60 CPU, 16 GB RAM, and RTX 3080 GPU.

\subsection{Benchmark results}

In this section, we first show that PKF can get comparable results to other methods~\cite{bytetrack,ocsort} by only associating bounding boxes, i.e., without using any other techniques or features. We then show that PKF is compatible with more advanced object features can serve as a backbone, same as SORT \cite{sort}.

\myparagraph{MOT17 and MOT20} The quantitative results on MOT17~\cite{MOT17} and MOT20~\cite{MOT20} are shown in Table~\ref{tab:mot17_results} and \ref{tab:mot20_results}. To get a fair comparison, we use the same detections as ByteTrack~\cite{bytetrack} and OC-SORT~\cite{ocsort} and inherit the linear interpolation. 
While the direct comparison to our PKF is the vanilla Kalman filter, we can achieve comparable results to ByteTrack~\cite{bytetrack} and OC-SORT~\cite{ocsort}, which do multiple rounds of associations~\cite{bytetrack} , make use of velocities~\cite{ocsort}, and re-updates~\cite{ocsort}.
We can see that PKF is able to achieve better data association, as indicated by IDF1 and AssA, and achieves slightly higher HOTA results. 


\begin{table}[t]
\centering
\caption{DanceTrack~\cite{dancetrack} results. \textcolor{cyan}{Blue} means shared detections.}
\label{tab:dancetrack_results}
\resizebox{\linewidth}{!}{
\begin{tabular}{l|lllll}
\toprule
Tracker                  & HOTA $\uparrow$ & DetA $\uparrow$ & AssA $\uparrow$ & MOTA $\uparrow$ & IDF1 $\uparrow$ \\ \midrule
FairMOT~\cite{zhang2021fairmot}      & 39.7 & 66.7 & 23.8 & 82.2 & 40.8 \\
QDTrack~\cite{QDTrack}               & 45.7 & 72.1 & 29.2 & 83.0 & 44.8 \\
TraDes~\cite{TraDes}                 & 43.3 & 74.5 & 25.4 & 86.2 & 41.2 \\
MOTR~\cite{MOTR}                     & 54.2 & 73.5 & 40.2 & 79.7 & 51.5 \\
\rowcolor{LightCyan} SORT~\cite{sort}         & 50.6 & 80.2 & 32.0 & 89.2 & 48.9 \\
\rowcolor{LightCyan} DeepSORT~\cite{deepsort}                 & 45.6 & 71.0 & 29.7 & 87.8 & 47.9 \\
\rowcolor{LightCyan} ByteTrack~\cite{bytetrack}                & 47.3 & 71.6 & 31.4 & 89.5 & 52.5 \\
\rowcolor{LightCyan} OC-SORT~\cite{ocsort}                  & 54.5 & 80.4 & 37.1 & 89.4 & 54.0 \\
\rowcolor{LightCyan} \textbf{PKF}                    & \textbf{53.7} & \textbf{79.7} & \textbf{36.2} & \textbf{89.1} & \textbf{53.5} \\
\bottomrule
\end{tabular}
}
\end{table}

\begin{table}[t]
\centering
\caption{Application to Hybrid-SORT-ReID~\cite{yang2024hybrid}.}
\label{tab:hybridpkf_results}
\resizebox{\linewidth}{!}{
\begin{tabular}{ll|llll}
\toprule
Dataset & Tracker                  & HOTA $\uparrow$ & AssA $\uparrow$ & MOTA $\uparrow$ & IDF1 $\uparrow$ \\ \midrule
\multirow{2}{*}{MOT20} & Hybrid-SORT-ReID~\cite{yang2024hybrid}                    & 63.9 & 64.7 & 76.7 & 78.4 \\
& Hybrid-PKF-ReID                   & \textbf{64.3} & \textbf{65.1} & 76.7 & \textbf{79.0} \\ \midrule
\multirow{2}{*}{DanceTrack} & Hybrid-SORT-ReID~\cite{yang2024hybrid}                    & 65.5 & 52.2 & \textbf{91.8} & 67.2 \\
& Hybrid-PKF-ReID                   & \textbf{66.4} & \textbf{53.9} & 91.6 & \textbf{68.7} \\
\bottomrule
\end{tabular}
}
\end{table}


\myparagraph{DanceTrack} The results on DanceTrack~\cite{dancetrack} are shown in Table~\ref{tab:dancetrack_results}. We can see that PKF achieves better results than SORT~\cite{sort} and ByteTrack~\cite{bytetrack}. 
In this dataset, there are a lot of occlusions and the dancers can move dramatically. The increase of AssA and IDF1 over SORT~\cite{sort} shows again the advantage of PKF in terms of association. 

\myparagraph{Application to Hybrid-SORT-ReID} Besides data association techniques from \cite{bytetrack,ocsort}, Hybrid-SORT-ReID~\cite{yang2024hybrid} further uses detection scores, a height-enhanced IoU, and a ReID module with deep neural features.
We show that our PKF is compatible with these features, just as SORT~\cite{sort} is, by replacing the SORT part (Kalman filter) with PKF. We name the new method Hybrid-PKF-ReID. With the score matrix computed by Hybrid-SORT-ReID using all these features, we compute the measurement probability as $p(\bfz_{t,k} | \bfx_{t,j}) \propto \exp(-\alpha /s_{k,j})$, where $s_{k,j}$ is the association score of $\bfz_{t,k}$ and $\bfx_{t,j}$. We test Hybrid-PKF-ReID on MOT20 and Dancetrack, which have more ambiguous scenarios. Table~\ref{tab:hybridpkf_results} shows that adding probabilistic association improves Hybrid-SORT-ReID's results. 
We show some sequences that PKF improves in Fig.~\ref{fig:dataset_examples}, which are highly occluded.
This suggests that PKF is effective in highly ambiguous cases.

\begin{figure}[t]
    \centering
    \includegraphics[width=\linewidth]{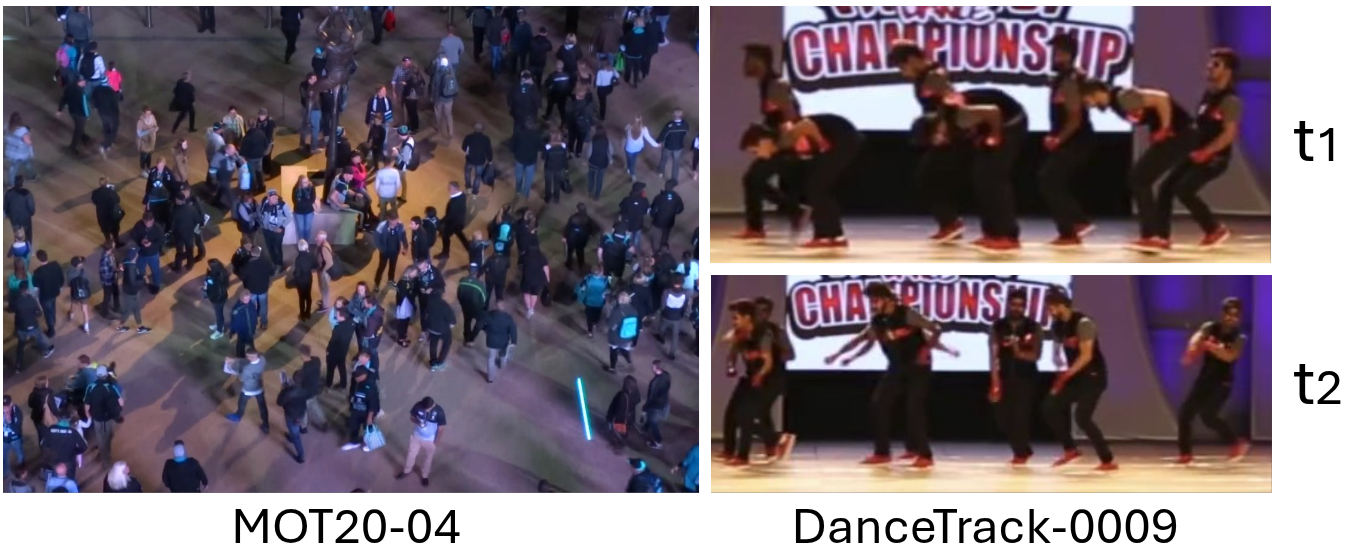}
    \caption{Visualizations of different datasets. 
    }
    \label{fig:dataset_examples}
\end{figure}


\begin{table}[t]
\centering
\caption{Ablation study of ambiguity check (A.C.). 
}
\label{tab:ablation}
\resizebox{\linewidth}{!}{
\begin{tabular}{c|ccc|ccc}
\toprule
                         & \multicolumn{3}{c|}{MOT17}       & \multicolumn{3}{c}{DanceTrack}   \\ 
\midrule
Metrics                  & HOTA $\uparrow$ & IDF1 $\uparrow$ & $t$ (ms) & HOTA $\uparrow$ & IDF1 $\uparrow$ & $t$ (ms) \\ 
\midrule
Binary                     & 64.9 & 76.9 & \textbf{5.9}               & 47.8 & 48.3 & \textbf{3.4}               \\
PKF w/o A.C. & 44.9    & 52.8    & 4970.6               & 44.2    & 45.2    & 121.2               \\
PKF & \textbf{66.8}    & \textbf{78.5}    & 6.2               & \textbf{53.5}    & \textbf{53.3}    & 3.7               \\ 
\bottomrule
\end{tabular}
}
\end{table}

\revision{
\begin{table}[!ht]
\centering
\caption{Effect of gating and ambiguity check in the JPDAF \cite{jpdaf} on the $l2$ tracking error (meters) in simulation.}
\label{tab:sim_ablation}
\resizebox{\linewidth}{!}{
\begin{tabular}{lllllllll}
\toprule
Method      & Obj 1 & Obj 2 & Obj 3 & Obj 4 & Obj 5 & Avg   \\ 
\midrule
Binary & 8.37     & 19.36    & 17.34     & 11.37     & 14.53     & 14.20  \\
JPDAF~\cite{jpdaf} w/o gating & 19.32     & 18.31     & 10.60     & 16.15        & 11.63        & 15.21  \\
JPDAF~\cite{jpdaf} w/ gating & 0.68     & \textbf{0.73}    & \textbf{0.62}     & \textbf{0.63}     & 0.66     & \textbf{0.66}  \\
JPDAF~\cite{jpdaf} w/ A.C.  & \textbf{0.62}     & 1.47     & 0.62     & 0.63     & \textbf{0.58}     & 0.79  \\ 
\bottomrule
\end{tabular}}
\end{table}

\begin{table}[!ht]
\centering
\caption{Computation time (ms) of matrix permanent and overall association weights with different matrix sizes.}
\label{tab:weight_computation_time}
\resizebox{\linewidth}{!}{
\begin{tabular}{lllllll}
\toprule
Matrix size     & $5\times5$      & $10\times10$ & $15\times15$ & $20\times20$ & $30\times30$  \\ \midrule
Nijenhuis et al.~\cite{nijenhuis2014combinatorial} & $0.003$    & $0.01$     & $2.0$     & $26.5$     & $31510.9$   \\
Huber et al.~\cite{huber2008fast}   & $6.5$ & $20.1$     & $21.7$     & $35.4$     & $67.8$  \\
Best   & $0.003$ & $0.01$  & $2.0$     & $26.5$     & $67.8$  \\ \midrule
Weight compute   & $0.03$ & $0.2$  & $14.5$     & $1073.6$     & $1766.6$  \\  
              \bottomrule
\end{tabular}}
\end{table}

}

\subsection{Ablation study}

We compare the effect of ambiguity checks and gating, and analyze the computation time of the matrix permanent.

\myparagraph{Ambiguity check}
We compare the performance of PKF with and without ambiguity check in Table~\ref{tab:ablation}. 
We observe that associating all measurements and tracks directly 
degrades the performance due to updating with low-weight measurements, which are mostly outliers. 
Introducing the ambiguity check in PKF to discard low-probability measurements before computing the matrix permanents, leads to better results than binary association with almost the same tracking speed. Table~\ref{tab:sim_ablation} shows that our ambiguity check has a similar effect to Chi-squared gating in the JPDAF \cite{jpdaf} with 5 objects in simulation. We do not do gating in real-world data since $1/IoU$ is ill-conditioned for Chi-squared test while the ambiguity check is more flexible with different features.

\myparagraph{Matrix permanent computation time}
We analyze the computation time of the matrix permanent used to obtained the association weights. We take the faster matrix permanent method between \cite{nijenhuis2014combinatorial} and \cite{huber2008fast} (switch at size $20\times20$). Table \ref{tab:weight_computation_time} shows that the speed is very fast for matrices smaller than $15\times15$ and computations on matrices up to $30\times30$ run at over $0.5$ fps.

\myparagraph{Discussion}
We observe that achieving effective use of probabilistic data association requires a balance between removing outliers (gating) and keeping in ambiguous measurements. If there is no gating, the outliers will bring bias to the measurements. On the other hand, if only the most likely measurements are considered, wrong hard decisions could harm the estimation. The probability and gating model can be improved to unlock further potential of PKF.
%


\section{Conclusion}
\label{sec:conclusion}

We derived a new formulation of the Kalman filter with probabilistic data association by formulating a variational inference problem and introducing data association as a latent variable in the EM algorithm. We showed that, in the E-step, the association probabilities can be computed as matrix permanents, while the M-step led to the usual Kalman filter prediction and update steps but with an extended measurement model. 
Our experiments demonstrated that our filter can outperform the JPDAF and can achieve good results on MOT benchmarks. Our algorithm is not restricted to the MOT application and can serve as a general method for estimation problems with measurement ambiguity.  


\bibliographystyle{ieeetr}
\bibliography{root}

\begin{thebibliography}{10}

\bibitem{sort}
A.~Bewley, Z.~Ge, L.~Ott, F.~Ramos, and B.~Upcroft, ``Simple online and
  realtime tracking,'' in {\em ICIP}, pp.~3464--3468, 2016.

\bibitem{bytetrack}
Y.~Zhang, P.~Sun, Y.~Jiang, D.~Yu, F.~Weng, Z.~Yuan, P.~Luo, W.~Liu, and
  X.~Wang, ``Bytetrack: Multi-object tracking by associating every detection
  box,'' in {\em ECCV}, pp.~1--21, Springer, 2022.

\bibitem{ocsort}
J.~Cao, J.~Pang, X.~Weng, R.~Khirodkar, and K.~Kitani, ``Observation-centric
  sort: Rethinking sort for robust multi-object tracking,'' in {\em CVPR},
  pp.~9686--9696, 2023.

\bibitem{yang2024hybrid}
M.~Yang, G.~Han, B.~Yan, W.~Zhang, J.~Qi, H.~Lu, and D.~Wang, ``Hybrid-sort:
  Weak cues matter for online multi-object tracking,'' in {\em Proceedings of
  the AAAI conference on artificial intelligence}, vol.~38, pp.~6504--6512,
  2024.

\bibitem{barfoot2020exactly}
T.~D. Barfoot, J.~R. Forbes, and D.~J. Yoon, ``Exactly sparse gaussian
  variational inference with application to derivative-free batch nonlinear
  state estimation,'' {\em IEEE IJRR}, vol.~39, no.~13, pp.~1473--1502, 2020.

\bibitem{Cao_MultiRobotSLAM_RAL24}
H.~Cao, S.~Shreedharan, and N.~Atanasov, ``{Multi-Robot Object SLAM Using
  Distributed Variational Inference},'' {\em IEEE Robotics and Automation
  Letters}, vol.~9, no.~10, pp.~8722--8729, 2024.

\bibitem{nijenhuis2014combinatorial}
A.~Nijenhuis and H.~S. Wilf, {\em Combinatorial algorithms: for computers and
  calculators}.
\newblock Elsevier, 2014.

\bibitem{huber2008fast}
M.~Huber and J.~Law, ``Fast approximation of the permanent for very dense
  problems,'' in {\em Proceedings of the nineteenth annual ACM-SIAM symposium
  on Discrete algorithms}, pp.~681--689, 2008.

\bibitem{pmht}
R.~L. Streit and T.~E. Luginbuhl, ``Maximum likelihood method for probabilistic
  multihypothesis tracking,'' in {\em Signal and data processing of small
  targets 1994}, vol.~2235, pp.~394--405, SPIE, 1994.

\bibitem{jpdaf}
Y.~Bar-Shalom, F.~Daum, and J.~Huang, ``The probabilistic data association
  filter,'' {\em IEEE Control Systems Magazine}, vol.~29, no.~6, pp.~82--100,
  2009.

\bibitem{dancetrack}
P.~Sun, J.~Cao, Y.~Jiang, Z.~Yuan, S.~Bai, K.~Kitani, and P.~Luo, ``Dancetrack:
  Multi-object tracking in uniform appearance and diverse motion,'' in {\em
  CVPR}, pp.~20993--21002, 2022.

\bibitem{IDF1}
E.~Ristani, F.~Solera, R.~Zou, R.~Cucchiara, and C.~Tomasi, ``Performance
  measures and a data set for multi-target, multi-camera tracking,'' in {\em
  ECCV}, pp.~17--35, Springer, 2016.

\bibitem{hota}
J.~Luiten, A.~Osep, P.~Dendorfer, P.~Torr, A.~Geiger, L.~Leal-Taix{\'e}, and
  B.~Leibe, ``Hota: A higher order metric for evaluating multi-object
  tracking,'' {\em IJCV}, vol.~129, pp.~548--578, 2021.

\bibitem{MOT17}
A.~Milan, L.~Leal-Taix{\'e}, I.~Reid, S.~Roth, and K.~Schindler, ``{MOT16}: A
  benchmark for multi-object tracking,'' {\em arXiv preprint arXiv:1603.00831},
  2016.

\bibitem{MOT20}
P.~Dendorfer, H.~Rezatofighi, A.~Milan, J.~Shi, D.~Cremers, I.~Reid, S.~Roth,
  K.~Schindler, and L.~Leal-Taix{\'e}, ``{MOT20}: A benchmark for multi object
  tracking in crowded scenes,'' {\em arXiv preprint arXiv:2003.09003}, 2020.

\bibitem{kaess2008isam}
M.~Kaess, A.~Ranganathan, and F.~Dellaert, ``isam: Incremental smoothing and
  mapping,'' {\em IEEE TRO}, vol.~24, no.~6, pp.~1365--1378, 2008.

\bibitem{jcbb}
J.~Neira and J.~D. Tard{\'o}s, ``Data association in stochastic mapping using
  the joint compatibility test,'' {\em IEEE Transactions on robotics and
  automation}, vol.~17, no.~6, pp.~890--897, 2001.

\bibitem{pdaf}
Y.~Bar-Shalom and E.~Tse, ``Tracking in a cluttered environment with
  probabilistic data association,'' {\em Automatica}, vol.~11, no.~5,
  pp.~451--460, 1975.

\bibitem{IPDA}
D.~Musicki, R.~Evans, and S.~Stankovic, ``Integrated probabilistic data
  association,'' {\em IEEE TAC}, vol.~39, no.~6, pp.~1237--1241, 1994.

\bibitem{JIPDA}
D.~Musicki and R.~Evans, ``Joint integrated probabilistic data association:
  Jipda,'' {\em IEEE transactions on Aerospace and Electronic Systems},
  vol.~40, no.~3, pp.~1093--1099, 2004.

\bibitem{blom2000probabilistic}
H.~A. Blom and E.~A. Bloem, ``Probabilistic data association avoiding track
  coalescence,'' {\em IEEE TAC}, vol.~45, no.~2, pp.~247--259, 2000.

\bibitem{yang2018linear}
S.~Yang, K.~Thormann, and M.~Baum, ``Linear-time joint probabilistic data
  association for multiple extended object tracking,'' in {\em IEEE Sensor
  Array and Multichannel Signal Processing Workshop (SAM)}, pp.~6--10, 2018.

\bibitem{rezatofighi2015joint}
S.~H. Rezatofighi, A.~Milan, Z.~Zhang, Q.~Shi, A.~Dick, and I.~Reid, ``Joint
  probabilistic data association revisited,'' in {\em ICCV}, pp.~3047--3055,
  2015.

\bibitem{meyer2018message}
F.~Meyer, T.~Kropfreiter, J.~L. Williams, R.~Lau, F.~Hlawatsch, P.~Braca, and
  M.~Z. Win, ``Message passing algorithms for scalable multitarget tracking,''
  {\em Proceedings of the IEEE}, vol.~106, no.~2, pp.~221--259, 2018.

\bibitem{survey}
L.~Rakai, H.~Song, S.~Sun, W.~Zhang, and Y.~Yang, ``Data association in
  multiple object tracking: A survey of recent techniques,'' {\em Expert
  systems with applications}, vol.~192, p.~116300, 2022.

\bibitem{bio_tracking}
W.~J. Godinez and K.~Rohr, ``Tracking multiple particles in fluorescence
  time-lapse microscopy images via probabilistic data association,'' {\em IEEE
  transactions on medical imaging}, vol.~34, no.~2, pp.~415--432, 2014.

\bibitem{bowman2017probabilistic}
S.~L. Bowman, N.~Atanasov, K.~Daniilidis, and G.~J. Pappas, ``Probabilistic
  data association for semantic slam,'' in {\em ICRA}, pp.~1722--1729, 2017.

\bibitem{doherty2019multimodal}
K.~Doherty, D.~Fourie, and J.~Leonard, ``Multimodal semantic slam with
  probabilistic data association,'' in {\em ICRA}, pp.~2419--2425, 2019.

\bibitem{faster_rcnn}
S.~Ren, K.~He, R.~Girshick, and J.~Sun, ``Faster r-cnn: Towards real-time
  object detection with region proposal networks,'' {\em NeurIPS}, vol.~28,
  2015.

\bibitem{yolo}
J.~Redmon, S.~Divvala, R.~Girshick, and A.~Farhadi, ``You only look once:
  Unified, real-time object detection,'' in {\em NeurIPS}, pp.~779--788, 2016.

\bibitem{hungarian}
H.~W. Kuhn, ``The hungarian method for the assignment problem,'' {\em Naval
  research logistics quarterly}, vol.~2, no.~1-2, pp.~83--97, 1955.

\bibitem{deepsort}
N.~Wojke, A.~Bewley, and D.~Paulus, ``Simple online and realtime tracking with
  a deep association metric,'' in {\em ICIP}, pp.~3645--3649, 2017.

\bibitem{feichtenhofer2017detect}
C.~Feichtenhofer, A.~Pinz, and A.~Zisserman, ``Detect to track and track to
  detect,'' in {\em ICCV}, pp.~3038--3046, 2017.

\bibitem{bergmann2019tracking}
P.~Bergmann, T.~Meinhardt, and L.~Leal-Taixe, ``Tracking without bells and
  whistles,'' in {\em CVPR}, pp.~941--951, 2019.

\bibitem{centertrack}
X.~Zhou, V.~Koltun, and P.~Kr{\"a}henb{\"u}hl, ``Tracking objects as points,''
  in {\em ECCV}, pp.~474--490, Springer, 2020.

\bibitem{centernet}
X.~Zhou, D.~Wang, and P.~Kr{\"a}henb{\"u}hl, ``Objects as points,'' {\em arXiv
  preprint arXiv:1904.07850}, 2019.

\bibitem{braso2020learning}
G.~Bras{\'o} and L.~Leal-Taix{\'e}, ``Learning a neural solver for multiple
  object tracking,'' in {\em CVPR}, pp.~6247--6257, 2020.

\bibitem{trackformer}
T.~Meinhardt, A.~Kirillov, L.~Leal-Taixe, and C.~Feichtenhofer, ``Trackformer:
  Multi-object tracking with transformers,'' in {\em CVPR}, pp.~8844--8854,
  2022.

\bibitem{transformer}
A.~Vaswani, N.~Shazeer, N.~Parmar, J.~Uszkoreit, L.~Jones, A.~N. Gomez,
  {\L}.~Kaiser, and I.~Polosukhin, ``Attention is all you need,'' {\em
  NeurIPS}, vol.~30, 2017.

\bibitem{bishop2006pattern}
C.~M. Bishop, {\em Pattern Recognition and Machine Learning}.
\newblock Springer, 2006.

\bibitem{montella2011kalman}
C.~Montella, ``The kalman filter and related algorithms: A literature review,''
  {\em Res. Gate}, pp.~1--17, 2011.

\bibitem{stonesoup}
S.~Hiscocks, J.~Barr, N.~Perree, J.~Wright, H.~Pritchett, O.~Rosoman,
  M.~Harris, R.~Gorman, S.~Pike, P.~Carniglia, L.~Vladimirov, and B.~Oakes,
  ``Stone {S}oup: No longer just an appetiser,'' in {\em International
  Conference on Information Fusion (FUSION)}, 2023.

\bibitem{zhang2021fairmot}
Y.~Zhang, C.~Wang, X.~Wang, W.~Zeng, and W.~Liu, ``Fairmot: On the fairness of
  detection and re-identification in multiple object tracking,'' {\em IJCV},
  vol.~129, pp.~3069--3087, 2021.

\bibitem{QDTrack}
J.~Pang, L.~Qiu, X.~Li, H.~Chen, Q.~Li, T.~Darrell, and F.~Yu, ``Quasi-dense
  similarity learning for multiple object tracking,'' in {\em CVPR},
  pp.~164--173, 2021.

\bibitem{MOTR}
F.~Zeng, B.~Dong, Y.~Zhang, T.~Wang, X.~Zhang, and Y.~Wei, ``{MOTR}: End-to-end
  multiple-object tracking with transformer,'' in {\em ECCV}, pp.~659--675,
  Springer, 2022.

\bibitem{chu2023transmot}
P.~Chu, J.~Wang, Q.~You, H.~Ling, and Z.~Liu, ``Transmot: Spatial-temporal
  graph transformer for multiple object tracking,'' in {\em WACV},
  pp.~4870--4880, 2023.

\bibitem{cai2022memot}
J.~Cai, M.~Xu, W.~Li, Y.~Xiong, W.~Xia, Z.~Tu, and S.~Soatto, ``{MeMOT}:
  Multi-object tracking with memory,'' in {\em CVPR}, pp.~8090--8100, 2022.

\bibitem{jonker-volgenant}
R.~Jonker and T.~Volgenant, ``A shortest augmenting path algorithm for dense
  and sparse linear assignment problems,'' in {\em DGOR/NSOR: Papers of the
  16th Annual Meeting of DGOR in Cooperation with NSOR/Vortr{\"a}ge der 16.
  Jahrestagung der DGOR zusammen mit der NSOR}, pp.~622--622, 1988.

\bibitem{clearmot}
K.~Bernardin and R.~Stiefelhagen, ``Evaluating multiple object tracking
  performance: the clear mot metrics,'' {\em EURASIP Journal on Image and Video
  Processing}, vol.~2008, pp.~1--10, 2008.

\bibitem{TraDes}
J.~Wu, J.~Cao, L.~Song, Y.~Wang, M.~Yang, and J.~Yuan, ``Track to detect and
  segment: An online multi-object tracker,'' in {\em CVPR}, pp.~12352--12361,
  2021.

\bibitem{stein1981estimation}
C.~M. Stein, ``Estimation of the mean of a multivariate normal distribution,''
  {\em The Annals of Statistics}, pp.~1135--1151, 1981.

\end{thebibliography}

\clearpage
\setcounter{page}{1}
\maketitlesupplementary

\section{Proof of Proposition~\ref{prop:weight_compute}}
\label{sec:proof_weight_compute}

Define sets $\calK^{-k}_t:=\{1, ..., M_t\} \backslash \{k\}$, $\calJ^{-j}:=\{1, ..., N\} \backslash \{j\}$, and a function $\bar{\delta}_t:\calK^{-k}_t \rightarrow \calJ^{-j}$. Then, \eqref{eq:weight_def} can be rewritten as:
\begin{align}
    w^{t}_{k,j} &\propto \!\!\!\! \sum_{\delta \in \calD_t(k,j)} \prod_{k'=1}^{M_t} p\left(\bfz_{t,k'} | \delta_t(k')\right) \notag \\
    &= \sum_{\bar{\delta}_t} p(\bfz_k | \delta_t(k)=j) \prod_{\bar{k} \in \calK^{-k}_t} p \left(\bfz_{t,\bar{k}} | \bar{\delta}_t(\bar{k}) \right) \nonumber \\ 
    &= p(\bfz_k | \delta_t(k)=j) \sum_{\bar{\delta}_t} \prod_{\bar{k} \in \calK^{-k}_t} p(\bfz_{t,\bar{k}} | \bar{\delta}_t(\bar{k})) \nonumber \\ 
    &= Q^{t}(k, j) \per(Q^{t}_{-kj}),
\end{align}
where the derivation in the last line follows by the definition of the matrix permanent in \eqref{eq:permanent_def}.

\section{Proof of Proposition~\ref{prop:KF}}
\label{sec:proof_PDA_KF}

In this section, we show how optimizing \eqref{eq:EM_problem} leads to Proposition~\ref{prop:KF}. The objective function $f^{(i)}(q)$ in \eqref{eq:EM_problem} can be rewritten as:
\begin{align}
    f^{(i)}(q)\! =& \bbE_{\delta_{t+1}} \bbE_{\bar{\bfx}^{(i)}} \!\! \left\{ \bbE_{q} \!\! \left[ \! \log  p(\bar{\bfx}, \bfz_{t+1}, \bfu_{t}, \delta_{t+1}) | \bar{\bfx}^{(i)}, \bfz_{t+1} \! \right] \! \right\}   \notag \\
    &- \bbE_{\delta_{t+1}} \bbE_{\bar{\bfx}^{(i)}} \left\{ \bbE_q \left[ \log q \right] \right\} \notag \\
    =& \bbE_q \!\! \left\{ \! \bbE_{\delta_{t+1}} \! \bbE_{\bar{\bfx}^{(i)}} \!\! \left[\! \log p(\bar{\bfx}, \bfz_{t+1}, \bfu_{t}, \delta_{t+1}) | \bar{\bfx}^{(i)}, \bfz_{t+1} \! \right] \! \right\} \notag \\    
    &\underbrace{- \bbE_q \left[ \log q \right]}_{\text{entropy}} \notag \\
    =& \bbE_q \left[J(\bar{\bfx}) \right] \underbrace{- \frac{1}{2} \log \left( | \Sigma^{-1} | \right) + c}_{\text{entropy}},
\end{align}
where $c$ is some constant unrelated to $q$. To find the maximizer of \eqref{eq:EM_problem}, we take the gradient of $f(q)$ w.r.t. $\bfmu$ and $\Sigma$ of $q$:
\begin{align}
    \frac{\partial f}{\partial \boldsymbol{\mu}^\top} &=   \Sigma^{-1} \bbE_q \left[ (\bar{\bfx} - \boldsymbol{\mu}) J(\bar{\bfx}) \right] \label{eq:first_grad_mu} \\
    \frac{\partial^2 f}{\partial \boldsymbol{\mu}^\top \partial \boldsymbol{\mu}} &=  \Sigma^{-1} \bbE_q \left[ (\bar{\bfx} - \boldsymbol{\mu}) (\bar{\bfx} - \boldsymbol{\mu})^\top J(\bar{\bfx}) \right] \Sigma^{-1} \notag \\
    & - \Sigma^{-1} \bbE_q \left[ J(\bar{\bfx}) \right] \label{eq:second_grad_mu} \\
    \frac{\partial f}{\partial \Sigma^{-1}} &= - \frac{1}{2} \bbE_q \left[ (\bar{\bfx} - \boldsymbol{\mu}) (\bar{\bfx} - \boldsymbol{\mu})^\top J(\bar{\bfx}) \right] \notag \\ 
    &+ \frac{1}{2}\Sigma \bbE_q[J(\bar{\bfx})] - \frac{1}{2} \Sigma \label{eq:first_grad_Sigma}
\end{align}
Using \eqref{eq:second_grad_mu} and \eqref{eq:first_grad_Sigma}, we get the following relationship:
\begin{equation}
    \frac{\partial^2 f}{\partial \boldsymbol{\mu}^\top \partial \boldsymbol{\mu}} = -\Sigma^{-1} - 2 \Sigma^{-1} \frac{\partial f}{\partial \Sigma^{-1}} \Sigma^{-1}.
    \label{eq:grad_mu_Sigma_relation}
\end{equation}
To compute the above derivatives, we use of Stein's Lemma~\cite{stein1981estimation}. With our notation, the Lemma says:
\begin{equation}
    \bbE_q [(\bar{\bfx} - \boldsymbol{\mu}) J(\bar{\bfx})] = \Sigma 
 \bbE_q \left[ \frac{\partial J}{\partial \bar{\bfx}} \right].
 \label{eq:stein_lemma}
\end{equation}
To compute $\frac{\partial J}{\partial \bar{\bfx}}$, we first analyze the term $J(\bar{\bfx})$:
\begin{align}
    J(\bar{\bfx}) \! &= \!  \mathbb{E}_{\delta_{t+1}} \! \bbE_{\bar{\bfx}^{(i)}} \!\! \left[\! \log  p(\bar{\bfx}, \bfz_{t+1}, \bfu_{t}, \delta_{t+1}) | \bar{\bfx}^{(i)}\!, \bfz_{t+1} \! \right] \label{eq:KF_objective_decom} \\
    &= \!\! \sum_{\delta \in \calD_{t+1}} \!\!\! \bbE_{\bar{\bfx}^{(i)}} \left[ p(\delta | \bar{\bfx}^{(i)}, \bfz_{t+1}) \right] \log p(\bar{\bfx}, \bfz_{t+1}, \bfu_t, \delta) \notag \\
    &\propto \!\! \sum_{\delta \in \calD_{t+1}} \!\!\! p(\delta | \bfz_{t+1}) \log (p(\bfz_{t+1} | \bar{\bfx}, \delta) p(\bar{\bfx} | \bfu_t)) \notag \\
    &= \!\! \underbrace{\sum_{\delta \in \calD_{t+1}} \!\!\! p(\delta | \bfz_{t+1}) \log p(\bfz_{t+1} | \bar{\bfx}, \delta)}_{J_{\bfz_{t+1}}(\bar{\bfx})} + \underbrace{\vphantom{ \sum_{\delta \in \calD_{t+1}}} \log p(\bar{\bfx} | \bfu_t)}_{J_{\bfu_t}(\bar{\bfx})}. \notag
\end{align}
Focusing on the first term on the right, we get:
\begin{align}
    J_{\bfz_{t+1}} &(\bar{\bfx}) =  \sum_{\delta \in \calD_{t+1}} p(\delta | \mathbf{z}_{t+1}) \log p(\mathbf{z}_{t+1} | \mathbf{x}_{t+1}, \delta) \\
    &=\! \sum_{k=1}^{M_t} \sum_{\delta \in \calD_{t+1}} \!\!\! p(\delta | \mathbf{z}_{t+1}) \log p(\mathbf{z}_{t+1, k} | \bfx_{t+1, \delta(k)}) \notag \\
    &= \! \sum_{k=1}^{M_t} \! \sum_{j=1}^N  \sum_{\delta \in \calD_{t+1}(k, j)} \!\!\!\!\!\!\!\!\! p(\delta | \bfz_{t+1}) \! \log p(\mathbf{z}_{t+1, k} | \bfx_{t+1, j}), \notag
\end{align}
where $\calD_{t+1}(k, j)$ is the subset of all possible data associations that assign measurement $\mathbf{z}_{t+1,k}$ to variable $\bfx_{t+1, j}$.
Let $w^{t+1}_{k,j} \triangleq \sum_{\delta \in \calD_{t+1}(k, j)} p(\delta | \bfz_{t+1})$, which is the same as \eqref{eq:weight_def}. The above equation can be rewritten as: 
\begin{align}
    J_{\bfz_{t+1}} (\bar{\bfx}) &= \sum_{k=1}^{M_t} \sum_{j=1}^N w^{t+1}_{k,j} \log(\mathbf{z}_{t+1, k} | \bfx_{t+1, j}) \label{eq:Jy(X)} \\
    &\propto \sum_{k=1}^{M_t} \sum_{j=1}^N w^{t+1}_{k,j} (-\frac{1}{2} \|\mathbf{z}_{t+1, k} - H \bfx_{t+1, j}\|_{V}) \notag \\
    &= \sum_{k=1}^{M_t} \sum_{j=1}^N -\frac{1}{2} \|\mathbf{z}_{t+1, k} - H \bfx_{t+1, j}\|_{V / w^{t+1}_{k,j}}^2. \notag
\end{align}
%
%
Then, we analyze the second term of \eqref{eq:KF_objective_decom}:
\begin{align}
    J_{\bfu_t} (\bar{\bfx}) &= \log p(\bar{\bfx} | \bfu_t) \label{eq:ju(X)} \\
    &= \log p(\bfx_t) p(\bfx_{t+1} | \mathbf{x}_{t}, \mathbf{u}_t) \notag \\
    &= \log p(\bfx_t) + \log p(\bfx_{t+1} | \bfx_t, \mathbf{u}_t) \notag \\
    &\propto -\frac{1}{2} \| \bfx_t - \boldsymbol{\mu}_t \|^2_{\Sigma_t} -\frac{1}{2} \| \bfx_{t+1} -  F \bfx_t - G \bfu_t \|^2_{W}. \notag
\end{align}
Combining \eqref{eq:Jy(X)} and \eqref{eq:ju(X)}, \eqref{eq:KF_objective_decom} can be expanded as:
\begin{align}
    J(\bar{\bfx}) \propto &\sum_{k=1}^{M_t} \sum_{j=1}^N -\frac{1}{2} \|\mathbf{z}_{t+1, k} - H \bfx_{t+1, j}\|_{V / w^{t+1}_{k,j}}^2 \label{eq:JX_ex} \\
    &-\frac{1}{2} \| \bfx_t - \boldsymbol{\mu}_t \|^2_{\Sigma_t} -\frac{1}{2} \| \mathbf{x}_{t+1} -  F \mathbf{x}_t - G \mathbf{u}_t \|^2_{W}. \notag
\end{align}
To further simplify the objective, we stack all the known data into a lifted column $\bfy$ as
\begin{equation}
    \bfy = \begin{bmatrix} \boldsymbol{\mu}_t \\ G \mathbf{u}_t \\ \Bar{\bfz}_{t+1} \\\end{bmatrix},  
\end{equation}
where $\bar{\bfz}_{t+1}$ is an expanded measurement defined as:
\begin{equation}
    \Bar{\mathbf{z}}_{t+1} = (I_{M_{t+1}} \otimes \mathbf{1}_{N} \otimes I_{m}) \bfz_{t+1} =
        \begin{bmatrix}
             \mathbf{z}_{t+1,1}  \\
             \vdots \\
             \mathbf{z}_{t+1,1}
        \\
        \vdots \\
             \mathbf{z}_{t+1,M_{t+1}}  \\
             \vdots \\
             \mathbf{z}_{t+1,M_{t+1}}
     \end{bmatrix} \!\!\! \in \mathbb{R}^{m M_{t+1} N}
\end{equation}
where $\bfz_{t+1,k} \in \bbR^{m}$ for all $k$, $M_{t+1}$ is the number of measurements at time $t+1$, and $N$ is the number of states. Define the following expanded measurement model:
\begin{equation}
    \Bar{\mathbf{z}}_{t+1} = \Bar{H}_{t+1} \bfx_{t+1} + \Bar{\bfv}_{t+1}, \quad \Bar{\bfv}_{t+1} \sim \mathcal{N}(\mathbf{0}, \Bar{V}_{t+1}),
    \label{eq:expanded_obs_model_ap}
\end{equation}
where $\Bar{V}_{t+1}$ is defined in \eqref{eq:expanded_obs_cov} (changing $t$ to $t+1$) and $\Bar{H}_{t+1}$ is defined as below
\begin{equation}
    \Bar{H}_{t+1} \!=\! \mathbf{1}_{M_{t+1}} \otimes \bfI_N \otimes H \!=\! 
        \begin{bmatrix}
        H   & \     & \     \\
        \   & \ddots& \     \\
        \   & \     & H     \\
        \   & \vdots& \     \\
        H   & \     & \     \\
        \   & \ddots& \     \\
        \   & \     & H     \\
        \end{bmatrix} \in \mathbb{R}^{m M_{t+1} N \times nN}.
\end{equation}
%
%
We, then, define the following block-matrix quantities:
\begin{equation}
    A = \begin{bmatrix}
        I   & \mathbf{0} \\
        -F  & I \\
        \mathbf{0} & \Bar{H}_{t+1} \\
        \end{bmatrix}, \ R = \begin{bmatrix} 
        \Sigma_t & \mathbf{0} & \mathbf{0}   \\
        \mathbf{0} & W & \mathbf{0}     \\
        \mathbf{0} & \mathbf{0} &  \Bar{V}_{t+1} \\
        \end{bmatrix}.
    \label{eq:KF_matrices}    
\end{equation}
%
%
Then, \eqref{eq:JX_ex} can be written in a matrix form as follows:
\begin{equation}
    J(\bar{\bfx}) = -\frac{1}{2} (\bfy - A\bar{\bfx})^T R^{-1} (\bfy - A\bar{\bfx}).
\end{equation}
%
The derivative $\frac{\partial J}{\partial \bar{\bfx} }$ is then
\begin{equation}
    \frac{\partial J(\bar{\bfx})}{\partial \bar{\bfx}}  = A^\top R^{-1} (\mathbf{y} - A\bar{\bfx}).
    \label{eq:dJ_dx}
\end{equation}
Using Stein's Lemma \eqref{eq:stein_lemma} and \eqref{eq:dJ_dx}, the first-order derivative $\frac{\partial f}{\partial \boldsymbol{\mu}^\top}$ \eqref{eq:first_grad_mu} becomes
\begin{equation}
    \frac{\partial f}{\partial \boldsymbol{\mu}^\top} = \bbE_q \left[ \frac{\partial J}{\partial \bar{\bfx}} \right] = - A^\top R^{-1} (\mathbf{y} - A\boldsymbol{\mu}).
    \label{eq:first_grad_mu_2}
\end{equation}
The second-order derivative $\frac{\partial^2 f}{\partial \boldsymbol{\mu}^\top \boldsymbol{\mu}}$ is then
\begin{equation}
    \frac{\partial^2 f}{\partial \boldsymbol{\mu}^\top \boldsymbol{\mu}} = -A^\top R^{-1} A.
\end{equation}
To get the $\Sigma$ minimizer, by setting $\frac{\partial f}{\partial \Sigma^{-1}} = 0$ and using \eqref{eq:grad_mu_Sigma_relation}, we get
\begin{align}
    \Sigma^{-1} =& -\frac{\partial^2 f}{\partial \boldsymbol{\mu}^\top \boldsymbol{\mu}} \label{eq:Sigma_inv} \\
    =& \begin{bmatrix}
        \Sigma^{-1}_t + F^\top W^{-1} F & - F^\top W^{-1} \\
        -W^{-1}F & \Bar{H}_{t+1}^\top \bar{V}_{t+1}^{-1} \Bar{H}_{t+1} + W^{-1}
        \end{bmatrix}. \notag
\end{align}
Since we want to obtain the marginal distribution $q(\bfx_{t+1})$, we need the bottom right block of the inverse of the above matrix. We make use of the following two Lemmas.
\begin{lemma}
    \label{lemma:block_matrix_inv}
    Let $S$ be a $2 \times 2$ block matrix, 
    \begin{equation}
        S = 
        \begin{bmatrix}
            A       & C \\
            C^\top  & B
        \end{bmatrix}.
    \end{equation}
    If $A$ and $B$ are invertible, the inverse of $S$ is
    \begin{align}
        &S^{-1} = \\ 
        &\begin{bmatrix}
            (A \! - \! C B^{-1} C^\top)^{-1} & -A^{-1} C (B \! - \! C^\top \!\! A^{-1} C)^{-1} \\
            -(B \! - \! C^\top \!\! A^{-1} C)^{-1} C^\top\!\! A^{-1} & (B \! - \! C^\top \!\! A^{-1} C)^{-1}
        \end{bmatrix}. \notag
    \end{align}
\end{lemma}
\begin{lemma}
    \label{lemma:schur_complement_inv}
    Given matrices $A\in \bbR^{N \times N}$, $B \in \bbR^{M \times M}$, $C \in \bbR^{N \times M}$, the inverse of $B  +  C^\top A^{-1} C$ is:
    \begin{equation}
        \begin{aligned}
        (B  +  &C^\top A^{-1} C)^{-1}  =  \\
        &B^{-1}  -  B^{-1} C^\top  (A  +  C B^{-1} C^\top)^{-1} C B^{-1}.
        \end{aligned}
    \end{equation}
\end{lemma}

Using Lemma~\ref{lemma:block_matrix_inv} and \eqref{eq:Sigma_inv}, we get
\begin{align}
    \Sigma^{-1}_{t+1} =&  \Bar{H}_{t+1}^\top \bar{V}_{t+1}^{-1} \Bar{H}_{t+1} + W^{-1} \label{eq:original_mariginal_cov} \\
    &- W^{-1} F (\Sigma^{-1}_t + F^\top W^{-1} F)^{-1} F^\top W^{-1}. \notag
\end{align}
Then, using Lemma~\ref{lemma:schur_complement_inv}, we get
\begin{align}
    \Sigma^{-1}_{t+1} &= (\Sigma^{+}_{t+1})^{-1} + \Bar{H}_{t+1}^\top \Bar{V}_{t+1}^{-1} \Bar{H}_{t+1}, \label{eq:final_maginal_cov} \\
    \Sigma^{+}_{t+1} &= F \Sigma_{t} F^\top + W, \label{eq:predicted_cov}
\end{align}
where $\Sigma^{+}_{t+1}$ is the predicted covariance.
Setting $\frac{\partial f}{\partial \boldsymbol{\mu}^\top} = 0$ in \eqref{eq:first_grad_mu_2}, we have
\begin{equation}
    \begin{aligned}
        \begin{bmatrix}
        \Sigma^{-1}_t + F^\top W^{-1} F & - F^\top W^{-1} \\
        -W^{-1}F & W^{-1}+\Bar{H}_{t+1}^\top \bar{V}_{t+1}^{-1} \Bar{H}_{t+1}
        \end{bmatrix}
        \begin{bmatrix}
        \boldsymbol{\mu}'_{t} \\
        \boldsymbol{\mu}_{t+1}
        \end{bmatrix}& \\
        = \begin{bmatrix}
            \Sigma_t^{-1} \boldsymbol{\mu}_{t} - F^\top W^{-1} G \mathbf{u}_t \\
            W^{-1} G \mathbf{u}_t + \Bar{H}_{t+1}^\top \Bar{V}_{t+1}^{-1}\Bar{\mathbf{z}}_{t+1}
        \end{bmatrix}&.
    \end{aligned}
    \label{eq:}
\end{equation}
Note that $\boldsymbol{\mu}'_{t}$ is the mean of $\bfs_t$ conditioned a future measurement $\Bar{\bfz}_{t+1}$, which is not suitable to compute in online filtering. We marginalize it by left-multiplying both sides by
\begin{equation}
    D=
    \begin{bmatrix}
    I & \mathbf{0} \\
    W^{-1} F (\Sigma_{t}^{-1} + F^\top W^{-1} F)^{-1} & I
    \end{bmatrix}.
\end{equation}
%
%
The simplified result is
\begin{equation}
    \begin{aligned}
        \begin{bmatrix}
        \Sigma_{t}^{-1} + F^\top W^{-1} F & -F^\top W^{-1} \\
        \mathbf{0} & \underline{(\Sigma^+_{t+1})^{-1} + \Bar{H}_{t+1}^\top \Bar{V}_{t+1}^{-1} \Bar{H}_{t+1}} 
        \end{bmatrix}
        \begin{bmatrix}
        \boldsymbol{\mu}'_{t} \\
        \boldsymbol{\mu}_{t+1}
        \end{bmatrix}& \\
        =  \begin{bmatrix}
            \Sigma^{-1}_{t} \boldsymbol{\mu}_{t} - F^\top W^{-1} G \mathbf{u}_t \\
            \uwave{(\Sigma^+_{t+1})^{-1}(F \boldsymbol{\mu}_{t} + G \mathbf{u}_{t}) + \Bar{H}^\top_{t+1} \Bar{V}_{t+1}^{-1} \Bar{\mathbf{z}}_{t+1}}
            \end{bmatrix}&.
    \end{aligned}
    \label{eq:simplified_result}
\end{equation}
While other terms are straightforward to compute, we provide explanation for the two terms marked as $\underline{(\cdot)}$ and $\uwave{(\cdot)}$. 
The term $\underline{(\cdot)}$ can be obtained in the same way as \eqref{eq:original_mariginal_cov}, \eqref{eq:final_maginal_cov}, \eqref{eq:predicted_cov}. 
%
Then, let us focus on the $\uwave{(\cdot)}$ term, which is the second block of $D A^\top R^{-1} \bfy$, denoted as $[D A^\top R^{-1} \bfy]_2$. We have
\begin{align}
    &[D A^\top R^{-1} \bfy]_2 = W^{-1} F (\Sigma_t^{-1} + F^\top W^{-1} F)^{-1} \Sigma^{-1}_t \boldsymbol{\mu} \nonumber \\
    &\quad + (W^{-1} + - W^{-1}F(\Sigma_t^{-1} + F^\top W^{-1} F) F^\top W^{-1}) G \bfu_t \nonumber \\
    &\quad + \Bar{H}^\top_{t+1} \Bar{V}_{t+1}^{-1} \Bar{\mathbf{z}}_{t+1}.
    \label{eq:vector_original}
\end{align}
The second term in \eqref{eq:vector_original} can be simplified in the same way as \eqref{eq:original_mariginal_cov}, \eqref{eq:final_maginal_cov}, \eqref{eq:predicted_cov} into $(\Sigma_{t+1}^+)^{-1} G \bfu_t$. The first term in \eqref{eq:vector_original} can be written as
\begin{align}
    &\quad W^{-1} F (\Sigma_t^{-1} + F^\top W^{-1} F)^{-1} \Sigma^{-1}_t \boldsymbol{\mu} \nonumber \\
    &= W^{-1} F [\Sigma_t - \Sigma_t F^\top (W + F \Sigma_t F^\top)^{-1} F \Sigma_t] \Sigma_t^{-1} \boldsymbol{\mu}_t \nonumber \\
    &= W^{-1} F [\Sigma_t - \Sigma_t F^\top (\Sigma^+_{t+1})^{-1} F \Sigma_t] \Sigma_t^{-1} \boldsymbol{\mu}_t \nonumber \\
    &= [W^{-1}  - W^{-1} F \Sigma_t F^\top (\Sigma^+_{t+1})^{-1}] F \boldsymbol{\mu}_t \nonumber \\
    &= [W^{-1} - W^{-1} (\Sigma^+_{t+1} - W) (\Sigma^+_{t+1})^{-1}] F \boldsymbol{\mu}_t \nonumber \\
    &= (\Sigma^+_{t+1})^{-1} F \boldsymbol{\mu}_t.
\end{align}
This way, \eqref{eq:vector_original} is simplified to the $\uwave{(\cdot)}$ term in \eqref{eq:simplified_result}.

The term $\boldsymbol{\mu}^+_{t+1} = F \boldsymbol{\mu}_{t} + G \mathbf{u}_{t}$ is the predicted mean. Bringing all of the above together, we have the recursive filtering as follows:
\begin{subequations}
	\label{eq:recur_filter}
	\begin{align}
		\text{Predictor}: \notag \\ \quad \Sigma^+_{t+1} &= F \Sigma_{t} F^\top + W, \tag{\ref{eq:recur_filter}{a}} \label{eq:recur_filter_a} \\
		\boldsymbol{\mu}^+_{t+1} &= F \boldsymbol{\mu}_{t} + G \mathbf{u}_t. \tag{\ref{eq:recur_filter}{b}} \label{eq:recur_filter_b} \\
		\text{Corrector}: \notag \\ \quad \Sigma_{t+1}^{-1} &= (\Sigma^+_{t+1})^{-1} + \Bar{H}_{t+1}^\top \Bar{V}_{t+1}^{-1} \Bar{H}_{t+1}, \tag{\ref{eq:recur_filter}{c}} \label{eq:recur_filter_c} \\
		\Sigma_{t+1}^{-1} \boldsymbol{\mu}_{t+1} &= (\Sigma^+_{t+1})^{-1} \boldsymbol{\mu}^+_{t+1} + \Bar{H}_{t+1}^T \Bar{V}_{t+1}^{-1} \Bar{\mathbf{z}}_{t+1}.  
	\tag{\ref{eq:recur_filter}{d}} \label{eq:recur_filter_d} 
	\end{align}
\end{subequations}
To obtain the canonical form, we first define the Kalman gain $\bar{K}_{t+1}$ as
\begin{equation}
    \bar{K}_{t+1} = \Sigma_{t+1} \bar{H}_{t+1}^\top \bar{V}_{t+1}^{-1},
\end{equation} 
then we do the following manipulation:
\begin{equation}
    \begin{aligned}
    I &= \Sigma_{t+1}\left[(\Sigma_{t+1}^+)^{-1} + \bar{H}_{t+1}^\top \bar{V}_{t+1}^{-1} \bar{H}_{t+1}\right] \\
    I &= \Sigma_{t+1} (\Sigma_{t+1}^+)^{-1} + \bar{K}_{t+1} \bar{H}_{t+1} \\
    \Sigma_{t+1} &= (I - \bar{K}_{t+1} \bar{H}_{t+1}) \Sigma_{t+1}^+ \\
    \bar{K}_{t+1} \bar{V}_{t+1} &= (I - \bar{K}_{t+1} \bar{H}_{t+1}) \Sigma_{t+1}^+ \bar{H}_{t+1}^\top.
    \end{aligned}
\end{equation}
Thus, we obtain the canonical Kalman gain
\begin{equation}
    \bar{K}_{t+1} = \Sigma^+_{t+1} \Bar{H}_{t+1}^\top (\Bar{H}_{t+1} \Sigma^+_{t+1} \Bar{H}_{t+1}^\top + \Bar{V}_{t+1})^{-1},
\end{equation}
and rewrite the recursive filter in canonical form:
\begin{subequations}
    \label{eq:KF_ap}
    \begin{align}
        \text{predictor:} \notag \\ \quad 
        \boldsymbol{\mu}^+_{t+1} &= F \boldsymbol{\mu}_{t} + G \mathbf{u}_t, \tag{\ref{eq:KF_ap}{a}} \label{eq:KF_ap_a} \\
        \quad  \Sigma^+_{t+1} &= F \Sigma_{t} F^T + W, \tag{\ref{eq:KF_ap}{b}} \label{eq:KF_ap_b} \\
        \text{corrector:} \notag \\ \quad \bfmu_{t+1} &= \bfmu_{t+1}^+ + \bar{K}_{t+1} (\bar{\bfz}_{t+1} - \bar{H}_{t+1} \bfmu_{t+1}^+),
        \tag{\ref{eq:KF_ap}{c}} \label{eq:KF_ap_c} \\
        \Sigma_{t+1} &= (I - \bar{K}_{t+1} \bar{H}_{t+1}) \Sigma_{t+1}^+. \tag{\ref{eq:KF_ap}{d}} \label{eq:KF_ap_d}
    \end{align}
\end{subequations}

\section{Motion and measurement model matrices}
\label{sec:motion_measurement_models}

In this section, we provide details for the motion and measurement matrices. For the motion model, since we store the velocities in the state, the motion model is $\bfx_{t+1} = [u_t + \dot{u}_t, v_t+\dot{v}_t, s_t+\dot{s}_t, r_t, \dot{u}_t, \dot{v}_t, \dot{s}_t ]$, and there are no control inputs, so we can set $G=0$. We have
\begin{align}
    F = I_{7\times7} + \begin{bmatrix}
        0_{3\times4} & I_{3\times3} \\
        0_{4\times4} & 0_{3\times3}
    \end{bmatrix},  G=0,  W = \begin{bmatrix}
        I_{4\times4} & 0_{4\times3} \\
        0_{3\times4} & 0.01 I_{3\times3}
    \end{bmatrix}. \notag
\end{align}
For the measurement model, the detections are converted into $\bfz = [u, v, s, r]^\top$, and we have
\begin{align}
    H = \begin{bmatrix}
        I_{4\times4} & 0_{4\times3}
    \end{bmatrix}, \; V = \begin{bmatrix}
        I_{2\times2} & 0_{2\times2} \\
        0_{2\times2} & 10 I_{2\times2}
    \end{bmatrix} \notag
\end{align}

\end{document}